\documentclass{article}
\usepackage[margin=1in]{geometry}
\usepackage{parskip}

\usepackage[round]{natbib}

\usepackage[scaled=0.9]{DejaVuSansMono}    %
\usepackage{listings}

\usepackage{xcolor}
\definecolor{keywordcolor}{HTML}{4169e1}   %
\definecolor{tacticcolor}{HTML}{4169e1}    %
\definecolor{commentcolor}{HTML}{2e8b57}   %
\definecolor{symbolcolor}{HTML}{000000}%
\definecolor{sortcolor}{HTML}{4169e1}      %
\definecolor{attributecolor}{HTML}{f75394} %
\definecolor{bgcolor}{gray}{0.95}

\usepackage{graphicx}
\usepackage{hyperref}%
\usepackage{amsmath,amssymb,amsthm}
\usepackage{bm}
\usepackage{mathtools}
\usepackage{proof}

\hypersetup{colorlinks=true, linkcolor=blue, citecolor=blue}

\usepackage[capitalize,noabbrev]{cleveref}
\AddToHook{env/proposition/begin}{\crefalias{theorem}{proposition}}
\AddToHook{env/lemma/begin}{\crefalias{theorem}{lemma}}
\AddToHook{env/corollary/begin}{\crefalias{theorem}{corollary}}
\AddToHook{env/definition/begin}{\crefalias{theorem}{definition}}
\AddToHook{env/assumption/begin}{\crefalias{theorem}{assumption}}
\AddToHook{env/example/begin}{\crefalias{theorem}{example}}
\AddToHook{cmd/appendix/before}{\crefalias{section}{appendix}}

\newcommand{\EE}{\mathbb{E}}
\newcommand{\PP}{\mathbb{P}}
\newcommand{\eps}{\varepsilon}

\newcommand{\calX}{\mathcal{X}}
\newcommand{\calY}{\mathcal{Y}}
\newcommand{\calA}{\mathcal{A}}
\newcommand{\calR}{\mathcal{R}}
\newcommand{\calS}{\mathcal{S}}
\newcommand{\calK}{\mathcal{K}}
\newcommand{\calL}{\mathcal{L}}
\newcommand{\calM}{\mathcal{M}}
\newcommand{\calD}{\mathcal{D}}

\DeclareMathOperator*{\argmin}{arg\,min}

\DeclareMathOperator{\tv}{TV}

\newcommand{\Latent}{\calM}
\newcommand{\gpred}{\mu}

\newcommand{\esucc}{\mathsf{Succ}}

\usepackage{mathtools}
\usepackage{proof}

\theoremstyle{plain}
\newtheorem{theorem}{Theorem}
\newtheorem{lemma}{Lemma}
\newtheorem{corollary}{Corollary}
\newtheorem{proposition}{Proposition}
\newtheorem*{namedthm*}{Theorem} %
\newtheorem*{namedlem*}{Lemma}   %
\theoremstyle{definition}
\newtheorem{definition}{Definition}
\newtheorem{assumption}{Assumption}
\newtheorem{example}{Example}
\theoremstyle{remark}

\title{\textbf{Exponential Sample Complexity Separation between Flat and Hierarchical Agentic Theorem Provers}}

\author{%
\textbf{Sho Sonoda}${}^{1,2}$ \hfill\small{\url{sho.sonoda@riken.jp}}\\
\textbf{Shunta Akiyama}${}^{1}$ \hfill\small{\url{akiyama_shunta@cyberagent.co.jp}}\\
\textbf{Yuya Uezato}${}^{1}$ \hfill\small{\url{uezato_yuya@cyberagent.co.jp}}\\
\small{${}^1$\textit{CyberAgent}}%
\small{${}^2$\textit{RIKEN AIP}}%
\phantom{\large{aaaaaaaaaaaaaaaaaaaaaaaaaaaaaaaaaaaaaaaaaaaaaaaaaaaaaaaaaa}}%
}

\date{May 7, 2026}

\begin{document}
\maketitle

\begin{abstract}
Agentic theorem provers often introduce intermediate lemmas, proof sketches, or subgoal decompositions before returning to tactic-level search. This can look like an expensive detour: if proving lemmas is itself hard, why should a learned prover spend effort there? We give a statistical learning answer. Instead of worst-case proof complexity over all formulas, we study the biased data distribution produced by a teacher prover: initial theorem states together with successful verified proof traces. We model proof search as a deterministic finite-horizon MDP and analyze offline imitation learning from those traces. The success bounds depend on the average length of teacher proofs, how predictable the teacher's next action is, and how accurately the student learns that local prediction problem. A flat student learns from fully inlined traces, so repeated subproofs appear many times in its training and test-time certificate. A hierarchical student instead predicts a reusable proof DAG and solves each shared block once. When flattening duplicates the same hard local argument exponentially many times, the sufficient-sample certificate produced by our bounds can be exponentially smaller for the hierarchical learner. This gives a concrete statistical mechanism by which reusable proof structure helps verifier-based theorem proving.
\end{abstract}

\section{Introduction}

Recent \emph{agentic theorem provers} combine an LLM with a verifier and often add decomposition steps such as lemma generation, proof sketches, or draft-refine loops
\citep{jiang2023draft,xin2024deepseekprover,Ren2025deepseekprover-v2,baba2025proveragent,bytedance2025seedprover,varambally2025hilbert,harmonic2025aristotle}.
A lemma or subgoal must itself be found and proved, so decomposition helps only if this extra planning cost is offset elsewhere.
The question we study is:
\emph{when can learning a reusable proof structure increase the probability of finding a verified proof within a fixed search budget?}

Classical proof complexity explains worst-case hardness and proof theory explains why lemmas or cuts can shorten derivations.
Modern learning-guided provers, however, are trained and evaluated on biased data: benchmarks, library theorems, and successful proofs generated by a particular prover or teacher.
We therefore study a statistical learning problem over the teacher-induced distribution of theorem states and successful proof traces.

The mechanism we analyze is simple.
If successful proofs repeatedly use the same intermediate argument, then a prover that represents this sharing can solve the hard subproof once and reuse it.
A flat prover that only sees fully inlined proof traces instead sees many copied occurrences of the same argument.
It may then need to learn and rediscover each copy as if it were a separate local decision.
Our goal is to turn this intuition into a finite-sample success guarantee.

We model theorem proving as a deterministic finite-horizon MDP and train by offline imitation learning from successful teacher proofs.
The success bounds depend on three plain quantities:
\begin{itemize}
\item the average length or complexity of the successful teacher proofs;
\item how locally predictable the teacher is, meaning whether the teacher's next action usually lies in a small shortlist of plausible actions;
\item how well the learned student predicts the teacher's next action.
\end{itemize}
For a flat prover these quantities are measured on fully inlined traces; for a hierarchical prover they split into graph choices and local proof steps.
If flattening a reusable proof DAG into a tree duplicates a subproof exponentially many times, then the sufficient-sample certificate produced by our bounds can be exponentially smaller for the hierarchical learner.

\paragraph{Why a statistical theory?}
For a fixed initial theorem state in a deterministic verifier-defined MDP, the best possible success probability over all policies is simply $0$ or $1$.
This is uninformative for learned provers, so we study \emph{statistical provability} \citep{Sonoda2026whyatp} rather than bare reachability.
Let $q_0$ be a distribution over initial theorem states.
We care about
\[
\mathrm{SP}_T(\pi):=\EE_{s_0\sim q_0}[V_T^\pi(s_0)],
\]
the average probability that the learned prover finds a verified proof within budget.
Here \(V_T^\pi\) and \(\mathrm{SP}_T(\pi)\) include the finite-list search wrapper, verifier transitions, and memoization when present; we suppress top-\(k\), beam, and budget parameters from the notation.

\paragraph{Contributions.}
\begin{itemize}
\item We formulate verifier-guided theorem proving as a finite-horizon MDP whose states are proof states and whose actions are tactics, inference-rule applications, or high-level decomposition moves.
\item We prove success guarantees for a flat student trained on inlined proof traces and for a hierarchical student that predicts reusable proof DAGs.
\item We show how the guarantees depend on average proof length, local predictability of the teacher, and student prediction error.
\item We compare flat and hierarchical students trained from the same hierarchical data and derive an exponential sufficient-sample separation when flattening destroys proof sharing.
\end{itemize}

\paragraph{Relation to prior work.}
We keep the verifier deterministic and ask how representing successful proof data as inlined traces or reusable DAGs changes a finite-budget statistical success guarantee.
\cref{sec:related-work} discusses learning-guided proving, proof-size phenomena, and hierarchical RL analogies.

\paragraph{Notational Convention.}
For measurable spaces $\calX,\calY$, write $p:\calX\leadsto\calY$ for a Markov kernel; equivalently, $p(\cdot\mid x)$ is a probability distribution on $\calY$.
This notation records the type of teachers, students, and graph predictors.

\section{Theorem Proving and Markov Decision Process}
\label{sec:proof-mdp}

This section connects proof systems with the MDP abstraction used in the analysis.
The high-level layer at the end is our modeling device for reusable proof blocks above the ordinary tactic or inference-rule level.

\newcommand{\flang}{\calL}
\newcommand{\axioms}{\calL_0}
\newcommand{\rules}{\calR}
\newcommand{\psystem}{\calK}

\paragraph{Ordinary proof systems.}
Let $\Sigma$ be a finite alphabet and let $\calL \subseteq \Sigma^\ast$ be a decidable or recursively enumerable formal language.
A formula is an element of $\calL$, and an $n$-ary inference rule is a relation $R \subseteq \calL^n \times \calL$.
A proof system over $\flang$ is a tuple
$\psystem=(\flang,\axioms,\rules)$,
where $\axioms\subseteq \flang$ is the set of axioms, and
$\rules$ is a family of inference rules.

Recursive applications of inference rules generate a proof tree.
For assumptions $\Gamma\subseteq \flang$ and conclusion $\varphi\in\flang$,
a $\psystem$-proof of $\varphi$ from $\Gamma$ is a finite rooted tree whose leaves are assumptions or axioms, whose internal nodes follow rules in $\rules$, and whose root is $\varphi$.
If there exists a $\psystem$-proof of $\varphi$ from $\Gamma$, we write
$\Gamma\vdash_\psystem \varphi$
and say 
that $\varphi$ is derivable/provable from $\Gamma$ in $\psystem$.
A formula $\varphi$ 
derivable without assumptions, i.e., $\emptyset\vdash_\psystem\varphi$, 
is called a theorem of $\psystem$.

For example, propositional logic and first-order logic specify object languages; natural deduction, LK, and Hilbert-style systems are proof systems for such languages.
LK is a Gentzen-style sequent calculus, while Hilbert-style systems use axiom schemas and few inference rules.

\paragraph{Proof assistants and proof search states.}
A proof assistant implements a formal language together with a trusted verifier.
Lean \citep{deMouraUllrich2021lean4} is a representative example: its kernel checks proof objects, while users and agents usually interact with a tactic interface.
At any time the visible Lean proof state is a finite collection of open goals with local contexts, abstractly a finite family of judgements
\[
\Gamma_i\vdash_{\psystem}\varphi_i .
\]
A tactic transforms this collection into a new collection of goals, and the proof is complete when it is empty.
The same abstraction covers sequent calculi and natural deduction: backward states store open goals, while forward states may store derived formulas and the target goal.
In both orientations, an action is an inference-rule instance, a tactic, a lemma call, or an instantiation, and the verifier determines whether the transition is valid.

\paragraph{Verifier MDP and provability.}
Fix an ordinary proof system $\psystem$.
We model AI-guided proof search over $\psystem$ by the deterministic verifier MDP
\[
\mathsf M_{\psystem}
=
(\calS,\calA_{\mathrm{lo}},F_{\mathrm{lo}},G,\iota).
\]
Here $\calS$ is the set of proof-search states, $\calA_{\mathrm{lo}}$ is the ordinary tactic alphabet, $F_{\mathrm{lo}}$ is the verifier-approved transition, $G\subseteq\calS$ is the set of solved states, and $\iota:\flang\to\calS$ maps a theorem statement to its initial state.
The adjective \emph{low-level} is not standard proof-theoretic terminology: in this paper it simply means the ordinary proof/tactic system before adding the high-level layer of Setting~2.

A learned prover is a Markov policy
\[
\pi:\calS\leadsto \calA_{\mathrm{lo}},
\]
combined with a fixed finite-candidate search wrapper.
For an initial state $s_0\in\calS$ and time budget $T$, its pointwise statistical provability is
\[
V_T^\pi(s_0):=\PP_{s_0}^\pi(\exists t\le T:\ s_t\in G).
\]
For a distribution $q_0$ over initial theorem states, its statistical provability is
\[
\mathrm{SP}_T(\pi;q_0)
:=
\EE_{s_0\sim q_0}[V_T^\pi(s_0)].
\]

\paragraph{High-level proof systems.}
Settings~2--3 add a nonstandard high-level layer on top of $\mathsf M_{\psystem}$.
A high-level proof object is a finite DAG
\[
M=(V,E,r,\lambda),
\]
where $r$ is the root block and each node $v\in V$ denotes a proof block with interface $\lambda(v)$, realized by a low-level proof fragment.
An edge $v\to w$ means that block $v$ may consume the result of block $w$; shared descendants are memoized subproofs.
The high-level system has states $\calS_{\mathrm{hi}}$ and graph actions $\calA_{\mathrm{hi}}$ for selecting, expanding, ordering, and reusing blocks.
These high-level actions are not ordinary tactics; ordinary tactics remain in $\calA_{\mathrm{lo}}$.
Soundness is enforced by an unfolding map from a high-level certificate and local block proofs to an ordinary low-level action string; typed details are in \cref{sec:app-hilevel-formal}.
Flattening means applying $\operatorname{Unf}$ while replacing every call to a shared block by a fresh copy of that block's low-level proof.
This sharing-versus-duplication distinction is the structural object analyzed in the rest of the paper.

\section{Common Assumptions}
\label{sec:setup}
Here we drop type $\tau \in \{\text{flat,hi,lo}\}$ for readability.

In the following, we compare three increasingly structured offline learning problems.
All three use a verifier-defined search problem, successful teacher proofs as training data, and finite-budget success probability as the evaluation criterion.
The difference is the proof representation available to the learner.
Setting~1 is fully flat.
Setting~2 learns a hierarchical prover from a hierarchical teacher.
Setting~3 uses the same hierarchical source data to compare a flat student against a hierarchical student.
All main guarantees use the success-conditioned problem distribution induced by the corresponding teacher protocol.

\paragraph{Quantities} %
Let $\calS$ and $\calA$ be its state and action spaces.
Let $q,\pi:\calS\leadsto\calA$ be the success-conditioned teacher kernel and the student kernel.
At a ranked state $s$, let
$\calA(s)\subseteq \calA$,
where $k+1\le |\calA(s)|<\infty$,
be the finite candidate set supplied by the search procedure.
All top-$k$ quantities are restricted to $\calA(s)$.
Let
$\mathrm{Top}_k(q(\cdot\mid s))$
be the set of the $k$ actions in $\calA(s)$ with largest teacher probabilities $q(a\mid s)$, with ties resolved by a fixed deterministic rule.
The teacher top-$k$ mass loss at state $s$ is
\[
\eta_k(s):=
1-\sum_{a\in \mathrm{Top}_k(q(\cdot\mid s))}q(a\mid s).
\]
This is the teacher probability mass already lost before student-learning error is introduced.
The average intrinsic search loss is $\eta:=\EE_{s\sim\calD}[\eta_k(s)]$, where $\calD$ is the success-filtered mixture of decision states specified in each setting.
The one-step imitation error is
\[
e(\pi):=
\EE_{s\sim\calD}\!\left[\mathrm{KL}(q(\cdot\mid s)\|\pi(\cdot\mid s))\right].
\]

\begin{assumption}[Tsybakov-type margin condition]
\label{ass:margin}
Let $\Delta_k(s)$ be the gap between the $k$th and $(k+1)$th teacher probabilities on $\calA(s)$.
There exist constants $C,\beta>0$ such that
\[
\PP(\Delta_k(s)\le u)\le C u^{\beta}
\qquad (u>0).
\]
Write $p=\beta/(\beta+2)$.
\cref{sec:app-proof-sketches} shows that this Tsybakov-type condition converts KL error into top-$k$ omission probability.
\end{assumption}

\begin{assumption}[Search completeness]
\label{ass:search-complete}
When every teacher action along a successful teacher proof is contained in the corresponding student top-$k$ list, the search procedure recovers a verified proof within the stated horizon.
With a bounded search budget $B$, an additional search-failure term $\zeta_B$ can be added to the failure bound.
In Setting~3, the horizon $T$ is assumed large enough to contain the relevant certificate representation being evaluated: the unfolded flat certificate for the flat student and the memoized hierarchical certificate for the hierarchical student.
If this is not true, the same bounds acquire the probability of exceeding the horizon as an additional failure term.
\end{assumption}

\begin{assumption}[Geometric learning condition]
\label{ass:geometry}
When this assumption is invoked for decision type $\tau$, the success-filtered $\tau$-trace length satisfies $0\le L_\tau\le L_{\tau,\max}$ and $\EE[L_\tau]\ge\mu_{\tau,\min}>0$.
The represented $\tau$-state space is compact with diameter $D_\tau$ and doubling dimension $d_\tau$.
The relevant action alphabet is finite, and every policy in $H_\tau$ has a probability floor $\rho_\tau>0$ on the relevant support and is Lipschitz in total variation.
For posterior-weighted hierarchical ERM, these conditions hold after conditioning on the fixed inference posterior, or under the sample-splitting convention used in \cref{thm:hier-geom-neurips}.
\end{assumption}

This condition is used only for learning curves; \cref{sec:app-seq} derives the occurrence-weighted random-denominator ERM bound.

\section{Setting 1: Flat Teacher and Flat Student}
\label{sec:setting-flat}
Here we drop type $\tau = \text{flat}$ (trivial!) for readability.

Setting~1 is the baseline: the teacher generates fully inlined low-level proof traces, the student imitates them, and evaluation runs flat test-time search without graph prediction or memoization.

\paragraph{Data Generation.}
Let $q_0$ be the base initial theorem distribution on $\calS_0\subset\calS$.
Let $\calS$ be the space of low-level proof states together with an occurrence identifier in an unfolded proof tree.
Let $\calA$ be a separate copy of the low-level tactic-action alphabet used at those occurrences.
A base flat teacher is a kernel
\[
q^{0}:\calS\leadsto\calA,
\]
and it induces a rollout law $\PP^{0}$ by sampling $s_0\sim q_0$, identifying $s_0$ with the root occurrence state in $\calS$, and then applying $q^{0}$ until horizon $T$.
Write
\[
\PP
:=
\PP^{0}(\,\cdot\mid \esucc\,),
\qquad
\esucc:=\{\exists t\le T:\ s_t\in G\},
\]
for the success-conditioned rollout law.
Under $\PP$, a training sample contains an initial theorem and a successful inlined trace
\[
y=((s_1,a_1),\ldots,(s_L,a_L)),
\qquad
s_i\in\calS,\ a_i\in\calA.
\]
Let $q_{0\mid\esucc}$ be the $s_0$-marginal of $\PP$.
The conditional one-step teacher kernel induced by $\PP$ is denoted $q$ in the top-$k$ quantities and imitation loss.
The decision-state mixture $\calD$ used in $\eta$ and $e$ below is the occurrence-weighted mixture:
for any test function $f$,
\[
\EE_{\calD}[f(s,a)]
:=
\frac{\EE_{\PP}
\left[\sum_{i=1}^{L(y)} f(s_i,a_i)\right]}
{\EE_{\PP}[L(y)]}.
\]

\paragraph{Learning Protocol.}
Let $H$ be the chosen flat hypothesis class, consisting of kernels
\[
\pi:\calS\leadsto\calA.
\]
Given $N$ independent successful traces, training is maximum likelihood:
\[
\hat\pi
\in
\argmin_{\pi\in H}
\frac{1}{\sum_{n=1}^N L_n}
\sum_{n=1}^N\sum_{i=1}^{L_n}-\log \pi(a_i^{(n)}\mid s_i^{(n)}).
\]
The action space $\calA$ contains tactic-level moves only:
there is no action for reusing a previously solved block, calling a memoized proof, or jumping to a shared graph node.

\paragraph{Test-Time Search.}
Given a fresh $s_0\sim q_{0\mid\esucc}$, the flat prover runs a verifier-guided search over low-level states only.
At each open state $s$, the fixed candidate generator returns $\calA(s)$, the policy $\pi(\cdot\mid s)$ ranks these candidates, and the search rule expands the selected finite shortlist through the deterministic transition $F_{\mathrm{lo}}$.
The search carries no memo table, so two copied occurrences of the same subgoal are explored as distinct flat states.
Success means reaching some state in $G$ within the horizon $T$ and the fixed search budget; the reported quantity is $\EE[V_T^{\pi}(s_0)]$.

\begin{theorem}[Setting 1 success guarantee]
\label{thm:flat-neurips}
Within this setting, write $L,\eta,e,p,c$ for the flat quantities associated with $q$, $\pi$, and $\calD$.
Assume \cref{ass:search-complete} for the Setting~1 test-time search and \cref{ass:margin}.
Then
\[
\EE_{s_0\sim q_{0\mid\esucc}}[V_T^{\pi}(s_0)]
\ge
1-
L\Bigl(\eta+ce(\pi)^p\Bigr).
\]
\end{theorem}

The bound separates the number of local decisions, the intrinsic top-$k$ loss, and the learned ranking error.

\begin{theorem}[Setting 1 geometric ERM rate]
\label{thm:flat-geom-neurips}
Assume the setting of \cref{thm:flat-neurips} and invoke \cref{ass:geometry} for the flat class $H$ with doubling dimension $d$ and probability floor $p_{\min}>0$.
If $\hat\pi_N$ is trained by ERM under log-loss, then with probability at least $1-\delta$,
\[
e(\hat\pi_N)
\le
\inf_{\pi\in H} e(\pi)
+
O\!\left(\Psi_d(N)+\sqrt{\frac{\log(1/\delta)}{N}}\right),
\quad\text{where}\quad
\Psi_d(N)\simeq
\begin{cases}
N^{-1/2}, & d=1,\\
N^{-1/2}\log N, & d=2,\\
N^{-1/d}, & d>2
\end{cases}
\]
up to polylogarithmic factors.
Substituting this estimate into \cref{thm:flat-neurips} gives the corresponding finite-sample success lower bound on $\EE_{s_0\sim q_{0\mid\esucc}}[V_T^{\hat\pi_N}(s_0)]$.
\end{theorem}

\section{Setting 2: Hierarchical Teacher and Hierarchical Student}
\label{sec:setting-hier}

Setting~2 introduces reusable proof structure.
The teacher chooses a high-level proof DAG and interleaves graph decisions with low-level local solving.
In the main setting, the learner observes only the inlined trace and trains through a posterior over latent proof DAGs.

\paragraph{Low-level proof graphs and high-level proof DAGs.}
A low-level proof graph records tactic-level states and verifier-approved actions.
A high-level proof DAG groups connected low-level fragments into proof blocks, with edges indicating block calls or reuse.

\paragraph{Data Generation.}
Let $q_0$ be the base initial theorem distribution on $\calS_0\subset\calS$, and let $\Latent$ be the space of high-level proof DAGs.
The base graph teacher is
\[
q_{\Latent}^{0}:\calS_0\leadsto\Latent.
\]
Given a graph, high-level search uses states $\calS_{\mathrm{hi}}$ that record the current graph frontier and the set of already solved blocks, with graph actions
$\calA_{\mathrm{hi}}$
such as selecting an unresolved block, expanding a block, or reusing a solved block.
Formally, $\calA_{\mathrm{hi}}$ is a disjoint union of graph-manipulation actions, not a low-level tactic alphabet; see \cref{def:app-hilevel-actions}.
The base high-level teacher is $q_{\mathrm{hi}}^{0}:\calS_{\mathrm{hi}}\leadsto\calA_{\mathrm{hi}}$.
Terminal blocks are solved on block-local states $\calS_{\mathrm{lo}}$ using low-level tactic actions $\calA_{\mathrm{lo}}$, with teacher
\[
q_{\mathrm{lo}}^{0}:\calS_{\mathrm{lo}}\leadsto\calA_{\mathrm{lo}}.
\]
Thus the teacher protocol is
\[
s_0 \xrightarrow{q_{\Latent}^{0}} M
\xrightarrow{q_{\mathrm{hi}}^{0}} \text{block schedule and reuse decisions}
\xrightarrow{q_{\mathrm{lo}}^{0}} \text{local verified proofs}.
\]
Together, $(q_0,q_{\Latent}^{0},q_{\mathrm{hi}}^{0},q_{\mathrm{lo}}^{0})$ induce a joint rollout law
\[
\PP_{\mathrm{hier}}^{0}
\quad\text{on}\quad
(s_0,M,y_{\mathrm{hi}},y_{\mathrm{lo}},y_{\mathrm{flat}}),
\]
where $y_{\mathrm{hi}}$ is the high-level decision trace, $y_{\mathrm{lo}}$ is the collection of local low-level traces for unique terminal blocks, and $y_{\mathrm{flat}}$ is the unfolded inlined trace.
Write
\[
\PP_{\mathrm{hier}}
:=
\PP_{\mathrm{hier}}^{0}(\,\cdot\mid\esucc\,),
\]
where $\esucc$ is the event that the generated hierarchical proof verifies within budget $T$.
Let $q_{0,\mathrm{hier}\mid\esucc}$ be the $s_0$-marginal of $\PP_{\mathrm{hier}}$.
The conditional one-step teacher kernels induced by $\PP_{\mathrm{hier}}$ are denoted $q_{\mathrm{hi}}$ and $q_{\mathrm{lo}}$ in the top-$k$ quantities and imitation losses.
The mixtures $\calD_{\mathrm{hi}}$ and $\calD_{\mathrm{lo}}$ are the analogous occurrence-weighted mixtures over high-level and low-level decision occurrences in $y_{\mathrm{hi}}$ and $y_{\mathrm{lo}}$.

\paragraph{Learning Protocol.}
Let $H_{\Latent}$, $H_{\mathrm{hi}}$, and $H_{\mathrm{lo}}$ be the chosen graph and policy classes.
The hierarchical student used at test time consists of a graph predictor and two policy kernels
\[
\gpred_\theta:\calS_0\leadsto\Latent,\qquad
\pi_{\mathrm{hi}}:\calS_{\mathrm{hi}}\leadsto\calA_{\mathrm{hi}},\qquad
\pi_{\mathrm{lo}}:\calS_{\mathrm{lo}}\leadsto\calA_{\mathrm{lo}}.
\]
In the main setting, graph annotations are not observed.
The training sample consists only of successful inlined traces
$(s_0^{(n)},y_{\mathrm{flat}}^{(n)})_{n=1}^N$.
The learner therefore uses a training-time inference posterior
\[
\nu_\phi:\calS_0\times\calY_{\mathrm{flat}}\leadsto\Latent
\]
over proof DAGs consistent with the full trace.
This posterior may use future states in the completed proof trace and is not available at test time.
An EM-style training step alternates between estimating or fixing $\nu_\phi$ and minimizing posterior-weighted graph, high-level, and low-level log-losses.
The graph predictor $\gpred_\theta$ is trained from the soft labels $\nu_\phi(\cdot\mid s_0,y_{\mathrm{flat}})$; with $\phi$ fixed during the policy update, the high- and low-level objective is
\[
\min_{\pi_{\mathrm{hi}},\pi_{\mathrm{lo}}}
\widehat R_{\mathrm{hi},N}^{\phi}(\pi_{\mathrm{hi}})
+
\widehat R_{\mathrm{lo},N}^{\phi}(\pi_{\mathrm{lo}}),
\]
where, for $\tau\in\{\mathrm{hi},\mathrm{lo}\}$,
\[
\widehat R_{\tau,N}^{\phi}(\pi_\tau)
:=
\frac{
\sum_{n=1}^N
\EE_{M\sim \nu_\phi(\cdot\mid s_0^{(n)},y_{\mathrm{flat}}^{(n)})}
\!\left[
\sum_{i=1}^{L_\tau(M,y_{\mathrm{flat}}^{(n)})}
-\log\pi_\tau(a_{\tau,i}^{(n,M)}\mid s_{\tau,i}^{(n,M)})
\right]}
{\sum_{n=1}^N
\EE_{M\sim \nu_\phi(\cdot\mid s_0^{(n)},y_{\mathrm{flat}}^{(n)})}
\!\left[L_\tau(M,y_{\mathrm{flat}}^{(n)})\right]}.
\]
The ideal posterior is $\PP_{\mathrm{hier}}(M\mid s_0,y_{\mathrm{flat}})$.
We write $\Delta_{\mathrm{post}}(\phi)$ for the graph-posterior mismatch and
$\Delta_{\mathrm{post},\tau}^{\mathrm{occ}}(\phi)$ for the occurrence-weighted posterior mismatch of the $\tau$-decision pair law; the precise definitions are in \cref{sec:app-proof-sketches}.

\paragraph{Test-Time Search.}
Given a fresh $s_0\sim q_{0,\mathrm{hier}\mid\esucc}$, the hierarchical prover first proposes one or more proof DAGs using $\gpred_\theta(\cdot\mid s_0)$.
For a proposed graph, the search maintains a high-level frontier and a memo table of solved blocks.
At a high-level state, $\pi_{\mathrm{hi}}$ ranks graph actions such as selecting an unresolved block, expanding it into child obligations, ordering the schedule, or reusing a compatible solved block.
When a terminal block must be discharged, the low-level search uses $\pi_{\mathrm{lo}}$ to rank tactic actions on block-local states and checks them by the ordinary verifier transition.
A solved block is stored in the memo table and can be consumed by later high-level actions without reproving its internal low-level trace.
Success means producing a high-level certificate whose unfolded low-level proof is verified within the horizon and search budget; the reported quantity is $\EE[V_T^{(\gpred_\theta,\pi_{\mathrm{hi}},\pi_{\mathrm{lo}})}(s_0)]$.

\paragraph{Graph coverage.}
A proposed high-level graph $\widehat M$ covers a teacher certificate $M$ if the high-level certificate used by $M$ embeds into $\widehat M$ by a map that preserves block interfaces, dependency edges, and the candidate sets needed for the teacher high-level actions.
For a graph predictor $\gpred$, define
\[
\Delta_{\mathrm{graph}}
:=
\PP_{(s_0,M)\sim\PP_{\mathrm{hier}},\ \widehat M\sim\gpred(\cdot\mid s_0)}
\!\left(\widehat M\text{ does not cover }M\right).
\]

\begin{assumption}[Graph prediction from latent labels]
\label{ass:graph-learning}
Let $e_{\Latent}(\gpred)$ be the excess graph log-loss under the true success-filtered law of $(s_0,M)$.
Viewed as a one-label trace problem, graph-label learning satisfies the geometric learning condition of \cref{ass:geometry}, with hypothesis class $H_{\Latent}$ and dimension $d_{\Latent}$; for posterior-weighted graph ERM, this is understood conditionally on the fixed posterior $\nu_\phi$ or under sample splitting.
In addition, graph prediction satisfies a low-noise conversion:
there exist $\zeta_{\Latent}\ge0$, $C_{\Latent}>0$, and $p_{\Latent}\in(0,1]$ such that
\[
\Delta_{\mathrm{graph}}(\gpred)
\le
\zeta_{\Latent}+C_{\Latent}e_{\Latent}(\gpred)^{p_{\Latent}}
\qquad(\gpred\in H_{\Latent}).
\]
\end{assumption}

The first clause gives the graph ERM learning curve used below, while the second converts graph log-loss into coverage error.
A multiclass Tsybakov margin for a dominant graph label is a sufficient condition for the conversion; $\zeta_{\Latent}$ is the irreducible coverage error of the best initial-state graph label.

\begin{assumption}[Reusable proof structure]
\label{ass:reusable}
In this paper, reusable proof structure means that successful proofs admit high-level DAG representations whose shared nodes correspond to proof blocks that are proved once and consumed by several later blocks.
Setting~2 assumes that such reusable proof DAGs occur often enough to be learned approximately.
The DAG is not assumed to be uniquely recoverable from raw syntax.
In LK it may correspond to cut-derived local blocks; in Lean it may correspond to \lstinline{have}-introduced lemmas, reusable \lstinline{calc}/\lstinline{refine} blocks, or local solver calls consumed multiple times later.
\end{assumption}

\begin{theorem}[Setting 2 success guarantee]
\label{thm:hier-neurips}
Let $L_{\mathrm{hi}}$ and $L_{\mathrm{lo}}$ be the average numbers of high-level and low-level decisions in a successful hierarchical proof.
Assume \cref{ass:search-complete} for the Setting~2 test-time search with memoization and \cref{ass:margin} for $\tau=\mathrm{hi},\mathrm{lo}$.
Then there exist constants $c_{\mathrm{hi}}$ and $c_{\mathrm{lo}}$ such that
\begin{align*}
\EE_{s_0\sim q_{0,\mathrm{hier}\mid\esucc}}[V_T^{(\gpred,\pi_{\mathrm{hi}},\pi_{\mathrm{lo}})}(s_0)]
\ge 1
&
-\Delta_{\mathrm{graph}}
-L_{\mathrm{hi}}\Bigl(\eta_{\mathrm{hi}}
+c_{\mathrm{hi}}e_{\mathrm{hi}}(\pi_{\mathrm{hi}})^{p_{\mathrm{hi}}}\Bigr)\\
&
-L_{\mathrm{lo}}\Bigl(\eta_{\mathrm{lo}}
+c_{\mathrm{lo}}e_{\mathrm{lo}}(\pi_{\mathrm{lo}})^{p_{\mathrm{lo}}}\Bigr).
\end{align*}
\end{theorem}

Unlike the flat bound, the hierarchical bound separates graph-level decomposition errors from local solving errors.
Hierarchy helps whenever it reduces the decision counts, sharpens the intrinsic top-$k$ distributions, or makes either local learning problem easier.

\begin{theorem}[Setting 2 latent-graph EM policy-learning rate]
\label{thm:hier-geom-neurips}
Assume the setting of \cref{thm:hier-neurips} and invoke \cref{ass:geometry} for $H_{\mathrm{hi}}$ and $H_{\mathrm{lo}}$ with doubling dimensions $d_{\mathrm{hi}}$ and $d_{\mathrm{lo}}$.
The observed training data are successful inlined traces, and the high-level graph $M$ is latent.
Condition on a fixed inference posterior $\nu_\phi$, or use sample splitting so that $\phi$ is independent of the samples used for the posterior-weighted policy update.
If the high- and low-level students are trained by posterior-weighted ERM under log-loss, then with probability at least $1-\delta$,
for each $\tau\in\{\mathrm{hi},\mathrm{lo}\}$,
\[
e_\tau(\hat\pi_N)
\le
\inf_{\pi\in H_\tau} e_\tau(\pi)
+O(r_{\tau,N}),
\]
where $r_{\tau,N}:=\Psi_{d_\tau}(N)+\sqrt{\log(1/\delta)/N}+\sqrt{\Delta_{\mathrm{post},\tau}^{\mathrm{occ}}(\phi)}$, and $\Psi_{d_\tau}$ is the same rate function presented in \cref{thm:flat-geom-neurips}.
This theorem controls only the high- and low-level policies; the test-time graph predictor is learned separately because $\nu_\phi$ takes the completed trace $y_{\mathrm{flat}}$ as input.
\end{theorem}

\begin{theorem}[Setting 2 graph-predictor rate]
\label{thm:graph-pred-neurips}
Assume \cref{ass:graph-learning}.
Train $\hat{\gpred}_N\in H_{\Latent}$ by posterior-weighted graph ERM using $\nu_\phi$ as soft labels.
Then with probability at least $1-\delta$,
\[
e_{\Latent}(\hat{\gpred}_N)
\le O(r_{\Latent,N}),
\qquad
\Delta_{\mathrm{graph}}(\hat{\gpred}_N)
\le \zeta_{\Latent}+C_{\Latent}O(r_{\Latent,N})^{p_{\Latent}},
\]
where
$r_{\Latent,N}:=\Psi_{d_{\Latent}}(N)+\sqrt{\log(1/\delta)/N}+\sqrt{\Delta_{\mathrm{post}}(\phi)}$.
\end{theorem}

Combining \cref{thm:hier-neurips,thm:hier-geom-neurips,thm:graph-pred-neurips} gives a full latent-graph bound for the test-time prover
$(\hat{\gpred}_N,\hat\pi_{\mathrm{hi},N},\hat\pi_{\mathrm{lo},N})$.

\begin{corollary}[Observed graph labels]
\label{cor:observed-graph-main}
If each successful training example includes the high-level graph label $M$, then $\nu_\phi$ may be replaced by the point mass $\delta_M$.
The posterior-mismatch terms vanish, and the graph predictor and both policies reduce to ordinary supervised ERM under the same geometric learning and graph low-noise assumptions.
\end{corollary}

\section{Setting 3: Same Hierarchical Data, Two Students}
\label{sec:setting-comparison}

Setting~3 is the comparison setting.
The source data are generated by the hierarchical teacher from Setting~2, but the learner may either discard or model the latent sharing.
This isolates the effect of the representation and learning rule while keeping the data source fixed.

\paragraph{Unfolding and inlined traces.}
Given a high-level proof DAG, its unfolding replaces every call to a shared high-level block by a fresh copy of that block's low-level proof subgraph.
The result is a tree-shaped occurrence graph.
An inlined trace is any fixed traversal order of all low-level decision occurrences in this unfolded tree:
\[
y_{\mathrm{flat}}
=
\bigl((s_1,a_1),\dots,(s_L,a_L)\bigr),
\qquad
s_i\in\calS_{\mathrm{flat}},\ a_i\in\calA_{\mathrm{flat}}.
\]
The occurrence component of $s_i$ is part of the flat state.
Thus two copied occurrences of the same syntactic subgoal are distinct states for the flat protocol.

\paragraph{Data Generation.}%
The source distribution is the success-conditioned hierarchical law $\PP_{\mathrm{hier}}$ from Setting~2, but the observed training sample given to both students is only
\[
(s_0^{(n)},y_{\mathrm{flat}}^{(n)})_{n=1}^N.
\]

\paragraph{Learning Protocols.}
The flat student discards the latent graph variable and uses the flat protocol from Setting~1 on the unfolded traces:
it trains $\pi_{\mathrm{flat}}$ by maximum likelihood on occurrence-level tactic decisions and has no memoization action.
The flat mixture $\calD_{\mathrm{flat}}$ in this comparison is therefore the decision-state mixture induced by the marginal law of unfolded traces under $\PP_{\mathrm{hier}}$.
The hierarchical student uses the latent-graph learning protocol from Setting~2:
it treats $M$ as latent, infers proof DAGs using $\nu_\phi$, trains $\gpred_\theta$ by posterior-weighted graph ERM, and trains $\pi_{\mathrm{hi}}$ and $\pi_{\mathrm{lo}}$ by posterior-weighted maximum likelihood.

\paragraph{Test-Time Search.}
Both students are evaluated on fresh initial states $s_0\sim q_{0,\mathrm{hier}\mid\esucc}$ from the same success-conditioned problem distribution.
The flat student runs the Setting~1 flat search on the unfolded low-level representation, so each copied occurrence must be rediscovered by occurrence-level tactic search.
The hierarchical student runs the Setting~2 search, using the test-time graph predictor $\gpred_\theta$ together with high-level and low-level policies with memoization.
The horizon is assumed large enough for both the unfolded flat certificate and the memoized hierarchical certificate; otherwise the corresponding horizon-overflow probability is added to the failure bound.

\begin{theorem}[Setting 3 sufficient trace counts]
\label{thm:sample-neurips}
Fix a target failure level $\delta\in(0,1)$ and write $L_{\mathrm{hier}}:=L_{\mathrm{hi}}+L_{\mathrm{lo}}$.
Assume
$L_{\mathrm{flat}}\eta_{\mathrm{flat}}\le \delta/2$, and 
$\Delta_{\mathrm{graph}}+L_{\mathrm{hi}}\eta_{\mathrm{hi}}+L_{\mathrm{lo}}\eta_{\mathrm{lo}}\le \delta/2$.
Assume learning-curve upper bounds in the number $N$ of independent successful theorem traces:
$e_{\mathrm{flat}}(\hat\pi_N)\le a_{\mathrm{flat}}N^{-\gamma_{\mathrm{flat}}}$ and
$e_{\tau}(\hat\pi_N)\le a_{\tau}N^{-\gamma_{\tau}}$ for $\tau\in\{\mathrm{hi},\mathrm{lo}\}$, after absorbing approximation and inference floors into constants.
Let $\bar N_{\mathrm{flat}}(\delta)$ and $\bar N_{\mathrm{hier}}(\delta)$ be the explicit certified sufficient thresholds obtained by solving the right-hand sides of \cref{thm:flat-neurips,thm:hier-neurips}.
They are certificates produced by these upper bounds, not claims about the unknown minimal sample complexities.
Then these certificates satisfy, up to constants and polylogarithmic factors,
\[
\bar N_{\mathrm{flat}}(\delta)
\lesssim
\left(L_{\mathrm{flat}}/\delta\right)^{1/(p_{\mathrm{flat}}\gamma_{\mathrm{flat}})},
\qquad
\bar N_{\mathrm{hier}}(\delta)
\lesssim
\max_{\tau\in\{\mathrm{hi},\mathrm{lo}\}}
\left(L_\tau/\delta\right)^{1/(p_\tau\gamma_\tau)}.
\]
In the common-exponent case $p_{\mathrm{flat}}\gamma_{\mathrm{flat}}=p_{\mathrm{hi}}\gamma_{\mathrm{hi}}=p_{\mathrm{lo}}\gamma_{\mathrm{lo}}=:p\gamma$, with comparable constants, the explicit flat and hierarchical certificates obey
\[
\bar N_{\mathrm{flat}}(\delta)/\bar N_{\mathrm{hier}}(\delta)
\gtrsim
\left(L_{\mathrm{flat}}/L_{\mathrm{hier}}\right)^{1/(p\gamma)}.
\]
Consequently, for any theorem family indexed by a size/depth parameter $D$ with fixed effective exponent $p\gamma$, if
\[
L_{\mathrm{flat}}/L_{\mathrm{hier}}=\exp(\Omega(D)), \quad\text{then}\quad
\bar N_{\mathrm{flat}}(\delta)/\bar N_{\mathrm{hier}}(\delta)\ge \exp(\Omega(D)).
\]
More generally, a polynomial blow-up in $L_{\mathrm{flat}}/L_{\mathrm{hier}}$ gives the corresponding polynomial blow-up in the certified sufficient sample size.
This is an exponential separation between the sufficient-sample certificates instantiated by the repeated-duplication mechanism in \cref{ex:two-m}.
\end{theorem}

\cref{thm:sample-neurips} compares sufficient sample sizes certified by \cref{thm:flat-neurips,thm:hier-neurips}, not minimax lower bounds.
It says that, under the stated protocols, unfolding a shared proof DAG can turn reusable decisions into many occurrence-level decisions and thereby separate the certified sample sizes.
Appendix~\ref{sec:app-matching-orders} states additional lower-bound assumptions under which the same scaling can be interpreted as a matching order for the actual sample threshold.
Flat learners with perfect canonicalization may remove this particular mechanism; separations for such learners require additional representation assumptions.

\section{Discussion and Conclusion}

The question raised in the introduction was why hierarchical proving should help when lemmas and subgoals must themselves be found and proved.
Our answer is statistical: if successful teacher proofs reuse the same intermediate argument, a hierarchical learner can represent it once, while a flat learner trained on inlined traces sees many copied occurrences.
The bounds express this through $L$, the average number of teacher decisions, $\eta$, the intrinsic top-$k$ mass loss or local predictability term, and $e$, the one-step imitation error.

The guarantee is meaningful when these error terms are small and the test-time graph predictor covers the teacher certificate with high probability.
It is not a worst-case proof-complexity lower bound and does not rule out flat learners with perfect canonicalization.
Rather, it identifies a concrete data-distribution-aware mechanism: when flattening destroys reusable proof structure, the hierarchical sufficient-sample certificate can be exponentially smaller for the same certified statistical provability.
A more detailed discussion is deferred to \cref{sec:aux-consequences}.

\subsubsection*{Acknowledgments}
\begin{small}
This work was supported by 
JST BOOST JPMJBY24E2, JST CREST JPMJCR25I5,
and JSPS KAKENHI 24K21316
\end{small}

\bibliographystyle{abbrvnat}
\bibliography{libraryS}

\appendix
\crefalias{section}{appendix} %

\section{Related Work}
\label{sec:related-work}

\paragraph{Learning-guided formal theorem proving.}
Long before the current LLM era, learning was used to guide premise selection, tactic choice, and proof search.
TacticToe learns tactic guidance for HOL4, while HOList provides a large-scale higher-order logic environment for imitation and reinforcement learning
\citep{Gauthier2017tactictoe,Gauthier2021,Bansal2019holist}.
Generative language models later made it natural to propose proof steps or whole proof fragments directly, as in GPT-f, curriculum-based formal mathematics, and HyperTree Proof Search
\citep{Polu2020gen,polu2023formal,Lample2022htps}.
Benchmarks and interfaces such as miniF2F and LeanDojo further shifted evaluation toward budgeted, distributional success of verifier-interactive systems
\citep{zheng2022miniff,yang2023leandojo}.
Our MDP model is compatible with this line of work, but our focus is narrower:
we ask how the representation of a successful proof, flat trace versus reusable DAG, changes the statistical burden of learning.

\paragraph{Agentic theorem provers and decomposition.}
Recent systems increasingly use informal sketches, intermediate lemmas, recursive decomposition, repair loops, or separate informal/formal roles.
Representative examples include Draft, Sketch, and Prove, DeepSeek-Prover-V2, Seed-Prover, Prover Agent, Hilbert, Aristotle, and Baldur
\citep{jiang2023draft,Ren2025deepseekprover-v2,bytedance2025seedprover,baba2025proveragent,varambally2025hilbert,harmonic2025aristotle,First2023baldur}.
Related work on autoformalization and mathematical language models studies how informal mathematical text or domain-specific pretraining can support formal proving
\citep{Wu2022autoformalization,azerbayev2024llemma,Lu2023mathreasoningsurvey,Li2024theoremprovingsurvey}.
These systems suggest that decomposition is empirically useful, but they do not by themselves isolate the statistical mechanism.
Our contribution is to show that, under a reusable proof-structure assumption, decomposition can reduce the number of learned decisions that must be simultaneously correct.

\paragraph{Lemmas, cut, and proof size.}
The proof-theoretic motivation is classical.
Cut and intermediate lemmas can make proofs much shorter, and eliminating cut can cause dramatic proof-size blow-up \citep{boolos1984dont-eliminate-cut}.
Modern automated reasoning also studies lemma generation, selection, and reuse as algorithmic objects \citep{Rawson2023lemmas}.
Our result does not prove a new worst-case lower bound for cut elimination.
Instead, it imports this structural lesson into a statistical setting:
if the data distribution contains proofs whose shared subarguments are duplicated by flattening, then the learning guarantee pays for occurrences in the flat representation but for unique reusable blocks in the hierarchical representation.

\paragraph{Statistical provability.}
Closest to our formulation is \citet{Sonoda2026whyatp}, which models agentic theorem proving as a finite-horizon MDP and defines statistical provability as the probability of reaching a verified proof under a problem distribution.
That work analyzes value learning and verifier-guided proving for a single learned prover, with error terms reflecting approximation, coverage, and rollout noise.
We instead study offline teacher--student imitation from successful proofs and compare two students trained from the same source:
a flat learner that imitates inlined traces and a hierarchical learner that infers a latent proof DAG.
This teacher--student comparison is what makes it possible to state an exponential sufficient-sample separation when flattening destroys sharing.

\paragraph{Reasoning models and test-time computation.}
Our paper is also related to theory for reasoning models.
\citet{Kim2025icml-reasoning} model chain-of-thought generation as a metastable Markov process, separating easy within-cluster moves from hard sparse transitions; this is conceptually close to our split between local solving and high-level decomposition.
Other work studies the expressive power or training of chain-of-thought, latent thought, and looped Transformers
\citep{kojima2022cot,feng2023cot-theory,phan2023cot,KevinXu2025icml-looped-trans,Saunshi2025latent-thoughts,KevinXu2026icml-cot,MerrillSabharwal2024cot,MerrillSabharwal2025padding}.
These results concern computational or approximation expressivity of reasoning architectures.
Our bounds concern a different object: sample complexity for verifier-certified proof search under a fixed top-$k$ success analysis.
Work on test-time scaling, best-of-$N$ policies, and verifier-assisted generation is also relevant because it studies how inference compute and verification interact
\citep{wu2024scaling,beirami2025bon,setlur2025tts-verifier,botta2025bon}.
In our setting the verifier is exact for formal correctness, but learned policies still decide which candidate actions receive scarce search budget.

\paragraph{Hierarchical RL and structured MDPs.}
Finally, there is a conceptual connection to hierarchical RL and structured MDPs.
Options, goal-conditioned hierarchy, and factored MDPs show that exploiting the right structure can improve sample efficiency
\citep{brunskill2014pac,fruit2017options,robert2023goalconditioned,kearns1999factored,strehl2007structure,osband2014factored,ziping2020oracle,chen2021factored}.
Our setting differs in two ways.
First, the transition map is the deterministic verifier rather than an unknown stochastic environment.
Second, the main structural primitive is not factorization of dynamics or temporal abstraction alone, but proof reuse:
a named lemma or proof block can be proved once and consumed many times.
This reuse is what turns a proof tree into a proof DAG and what drives the flat-versus-hierarchical separation in our analysis.

\section{Further Discussion}
\label{sec:aux-consequences}
\label{sec:concrete}

\subsection{When Are the Guarantees Informative?}
\label{sec:app-guarantee-large}

The bounds in Settings~1--3 are informative only when the one-step failure terms remain small after multiplication by the relevant decision counts.
For the flat setting this means
\[
L_{\mathrm{flat}}\bigl(\eta_{\mathrm{flat}}+c_{\mathrm{flat}}e_{\mathrm{flat}}^{p_{\mathrm{flat}}}\bigr)\ll 1.
\]
For the hierarchical setting this means
\[
\Delta_{\mathrm{graph}}
+
L_{\mathrm{hi}}\bigl(\eta_{\mathrm{hi}}+c_{\mathrm{hi}}e_{\mathrm{hi}}^{p_{\mathrm{hi}}}\bigr)
+
L_{\mathrm{lo}}\bigl(\eta_{\mathrm{lo}}+c_{\mathrm{lo}}e_{\mathrm{lo}}^{p_{\mathrm{lo}}}\bigr)
\ll 1.
\]
Thus the success probability becomes large when four things hold simultaneously:
the teacher distribution is locally concentrated, the student is accurate enough to preserve the teacher's top-$k$ actions, the test-time graph predictor covers the reusable certificate with high probability, and the representation has a small enough effective decision count.
In the latent-graph version of Setting~2, the graph term is itself learned:
\cref{thm:graph-pred-neurips} gives
\[
\Delta_{\mathrm{graph}}(\hat{\gpred}_N)
\le
\zeta_{\Latent}
+
C_{\Latent}O(r_{\Latent,N})^{p_{\Latent}},
\]
so the hierarchical guarantee is informative only when the dominant graph label is predictable from the initial theorem state and the posterior mismatch term in $r_{\Latent,N}$ is small.
Setting~3 isolates the last point:
when flattening duplicates a reusable proof block, $L_{\mathrm{flat}}$ can be exponentially larger than $L_{\mathrm{hier}}$, even if the local learning quality is comparable.

\subsection{Intrinsic Top-\texorpdfstring{$k$}{k} Loss and Small Beams}
\label{sec:app-small-k}

\paragraph{When is $\eta$ small and when is it large?}
The constant $\eta$ does not decay with training; it describes the problem distribution itself.
It is small when the teacher already has a concentrated local ranking.
For example, in backward proof search for propositional logic, if a state has an essentially forced introduction rule $a^\star$ with
\[
q(a^\star\mid s)\ge 1-\eps,
\]
then $\eta_1(s)\le \eps$.
The same phenomenon appears in Lean states closed by routine moves such as \lstinline{intro}, \lstinline{rfl}, \lstinline{exact h}, \lstinline{simp}, or \lstinline{linarith}.
By contrast, $\eta$ becomes large when many semantically different next steps compete with similar plausibility.
Examples include states requiring creative lemma invention, ambiguous first-order instantiations, or broad branching points where several decompositions look equally viable to the teacher.
In those regimes the problem is intrinsically harder even before one asks how well the student generalizes.

\paragraph{When can a small $k$ work?}

Two simple observations explain when small $k$ is enough.

\begin{definition}[Effective rank]
Fix a decision type $\tau\in\{\mathrm{flat},\mathrm{hi},\mathrm{lo}\}$.
For $\eps\in(0,1)$, define
\[
r_{\tau,\eps}(s):=\min\{k:\eta_{\tau,k}(s)\le \eps\}.
\]
This is the smallest list size that captures teacher mass at least $1-\eps$.
\end{definition}

\begin{proposition}[Teacher concentration implies small $k$]
If for some $k_0,\eps,\zeta$,
\[
\PP(r_{\tau,\eps}(s)\le k_0)\ge 1-\zeta,
\]
then
\[
\EE[\eta_{\tau,k_0}(s)]\le \eps+\zeta.
\]
\end{proposition}

\begin{proof}
On the event $\{r_{\tau,\eps}(s)\le k_0\}$ we have $\eta_{\tau,k_0}(s)\le \eps$.
Outside that event, trivially $\eta_{\tau,k_0}(s)\le 1$.
Take expectations.
\end{proof}

\begin{proposition}[Pretraining can make a small $k$ usable]
Suppose a fixed $k$ and decision type $\tau$ satisfy $\EE[\eta_{\tau,k}(s)]\le \eps_{\mathrm{bias}}$ and the learned policy satisfies
\[
e_\tau(\pi_\tau)\le \left(\frac{\eps_{\mathrm{learn}}}{c(C_\tau,\beta_\tau)}\right)^{1/p_\tau}.
\]
Then the one-step failure term in the main bounds is at most
\[
\eps_{\mathrm{bias}}+\eps_{\mathrm{learn}}.
\]
\end{proposition}

This is the operational meaning of ``a good pretrained prover needs only a small beam''.
Learning does not reduce $k$ directly.
Rather, the teacher must already be concentrated enough that small $k$ has low bias, and the student must then be accurate enough not to disturb that ranking.

\subsection{Reusable Proof Families}
\label{sec:app-reusable-proofs}

\paragraph{What concrete proofs fit the hierarchical assumption?}
The cleanest examples are proofs with repeated local arguments.
One proves an intermediate claim once, stores it as a named lemma or proof block, and reuses it in several later branches.
This is common in sequent-style reasoning and in Lean proofs built around \lstinline{have}, local lemmas, or reusable solver calls.
The assumption is not that the entire proof be tree-shaped at a high level.
It is only that there exists a decomposition into reusable blocks whose shared representation is shorter than the fully inlined trace.

\paragraph{Examples that make both $\eta_{\mathrm{hi}}$ and $\eta_{\mathrm{lo}}$ small.}
At the high level, many proof states have an almost obvious next decomposition: introduce an auxiliary claim, split a conjunction, or invoke a previously used lemma schema.
At the low level, terminal blocks often collapse to solver-like closures such as rewriting, arithmetic normalization, or direct hypothesis discharge.
This is precisely the regime where hierarchy helps twice: it reduces the number of decisions and also sharpens the local teacher distributions at each level.

\paragraph{Why can hierarchy also make learning easier?}
The separation theorem only needs a shorter effective proof representation.
But in practice hierarchy can help twice.
High-level policies only need to choose decompositions or intermediate claims.
Low-level policies only need to solve localized leaf problems.
Both tasks can be simpler than modeling a single monolithic next-step distribution over all proof states.
In Setting~2 there is also a graph-prediction task.
If the data distribution has a low-entropy graph label given the initial theorem state, then $\Delta_{\mathrm{graph}}(\hat{\gpred}_N)$ can be small; otherwise graph prediction becomes the bottleneck even if the local high- and low-level policies are accurate.
Faster decay of $e_{\mathrm{hi}}$ and $e_{\mathrm{lo}}$ than $e_{\mathrm{flat}}$ strengthens the advantage, but the sufficient-sample separation in Setting~3 only requires the shorter effective representation together with small graph-coverage error.

\subsection{Logics and Proving Styles Covered}
\label{sec:app-covered-logics}

\paragraph{Which logics are covered?}
The framework is deliberately syntax-agnostic.
A state may be a finite collection of open goals together with local context; an action may be an inference rule, tactic, lemma call, or instantiation; and the deterministic transition is the proof assistant kernel or deduction rule application.
The success set consists of closed proof states.
This covers propositional sequent calculi and natural deduction, first-order systems such as LK and LJ, backward-reasoning automated theorem proving, and interactive theorem proving in systems such as Lean or Coq.
The action space itself may be infinite.
The main requirement is only that, at each ranking step, the search procedure supplies a finite candidate set $\calA(s)$ on which top-$k$ selection is performed.
Proof-theoretic provability is the existence of a finite derivation, whereas our statistical provability is the probability that the learned, budgeted search reaches such a derivation within horizon $T$.

\section{Success-Conditioned Teacher Kernels}
\label{sec:data-generation-methods}

The main text defines training and evaluation laws by conditioning a base
teacher rollout on successful verification within the horizon.  This appendix
records the standard construction that turns a base teacher into the
success-conditioned one-step kernels used in the proofs.

\subsection{Doob transform for hitting a solved state}

Consider a deterministic verifier transition
\[
F:\calS\times\calA\to\calS
\]
and a base teacher kernel $q^0:\calS\leadsto\calA$.
As in the main text, a rollout stops after reaching $G$; equivalently, add an
absorbing solved state with a fixed self-loop action.
For remaining horizon $r\ge 0$, define the success-to-go function
\[
h_r(s):=\PP_s^{q^0}(\exists t\le r:\ S_t\in G).
\]
Then $h_0(s)=\mathbf 1\{s\in G\}$ and, for $r\ge 0$,
\[
h_{r+1}(s)
=
\mathbf 1\{s\in G\}
+
\mathbf 1\{s\notin G\}
\int_{\calA} h_r(F(s,a))\,q^0(da\mid s).
\]

\begin{definition}[Doob-transformed teacher]
\label{def:doob-success-teacher}
For $r\ge 1$ and $s\notin G$ with $h_r(s)>0$, define
\[
q_r^\star(a\mid s)
:=
\PP^{q^0}(A_0=a\mid S_0=s,\ \exists t\le r:\ S_t\in G)
=
\frac{q^0(a\mid s)h_{r-1}(F(s,a))}{h_r(s)}.
\]
For $s\in G$, take any fixed self-loop action convention.
\end{definition}

\begin{proposition}[Exact conditional sampling]
\label{prop:doob-conditional-sampling}
Fix a horizon $T$ and an initial state $s_0$ with $h_T(s_0)>0$.
Generate a trajectory by
\[
A_t\sim q_{T-t}^\star(\cdot\mid S_t),
\qquad
S_{t+1}=F(S_t,A_t),
\qquad t=0,\ldots,T-1.
\]
Then the law of the generated trajectory up to time $T$ is the law of the
base $q^0$ trajectory conditioned on the event
$\{\exists t\le T:\ S_t\in G\}$.
\end{proposition}

\begin{proof}
For an action sequence $a_{0:T-1}$ with induced states
$s_{t+1}=F(s_t,a_t)$ and with no solved state before time $T$ except
possibly at the endpoint, the transformed path probability is
\[
\prod_{t=0}^{T-1}
\frac{q^0(a_t\mid s_t)h_{T-t-1}(s_{t+1})}{h_{T-t}(s_t)}
=
\frac{\prod_{t=0}^{T-1}q^0(a_t\mid s_t)h_0(s_T)}{h_T(s_0)}.
\]
If the path hits $G$ earlier, the same telescoping argument applies after
stopping at the first hitting time, using the self-loop convention on solved
states.
Thus the transformed path probability is the base path probability multiplied
by the indicator of success within horizon $T$ and divided by
$h_T(s_0)=\PP_{s_0}^{q^0}(\exists t\le T:\ S_t\in G)$.
This is exactly the conditional path law.
\end{proof}

\subsection{Relation to the main text}

In Setting~1, applying the construction to $q_{\mathrm{flat}}^0$ gives the
success-conditioned flat decision kernel denoted $q_{\mathrm{flat}}$ in the
top-$k$ loss and imitation error.
In Setting~2, the same construction applies to the joint hierarchical rollout
over graph choices, high-level choices, and local low-level choices; the
regular conditional one-step kernels under that path law are denoted
$q_{\mathrm{hi}}$ and $q_{\mathrm{lo}}$.

The transform is used only to justify the conditional kernels mathematically.
Practical theorem provers need not compute $h_r$ exactly; rejection sampling,
curated successful proof corpora, or approximate twisted particle methods can
all be viewed as ways to approximate the same success-conditioned path law.

\section{Proofs for the Main Theorems}
\label{sec:app-proof-sketches}

\subsection{Standing conventions}

The body states the success theorems for the success-conditioned problem distribution.
Equivalently, the initial theorem distribution in each theorem is $q_{0\mid\esucc}$, the base initial distribution conditioned on the event that the corresponding teacher protocol produces a verified proof within the horizon.
If the original, unconditioned teacher success probability is $\alpha$, then the same proof gives the conditional guarantee below and the unconditional lower bound is multiplied by $\alpha$ unless one adds a separate guarantee on the teacher-failure slice.
The one-step kernels $q_\tau$ appearing in $\eta_\tau$ and $e_\tau$ are the regular conditional teacher kernels under this success-conditioned law.
\cref{sec:data-generation-methods} explains how such kernels arise from a base teacher by a Doob transform.

All top-$k$ sets are computed inside the finite candidate set $\calA_\tau(s)$, and the relevant teacher and student kernels are supported on this set.
Ties are resolved by a fixed deterministic rule.
For a distribution $p$ on $\calA_\tau(s)$, write $T_k(p)$ for the resulting top-$k$ set.
For each decision type $\tau$, the mixture $\calD_\tau$ is occurrence-weighted, so that for any nonnegative function $h$ on decision states,
\[
\EE\!\left[\sum_{i=1}^{L_\tau} h(S_i)\right]
=
L_\tau\,\EE_{S\sim\calD_\tau}[h(S)],
\]
where $L_\tau$ is the average number of $\tau$-decisions in the successful teacher certificate.
Throughout the learning-rate statements, $N$ denotes the number of independent successful theorem traces; the horizon and length factors are tracked separately in the success bounds or absorbed into the constants of the trace-level generalization statement.

\subsection{From KL error to \texorpdfstring{top-$k$}{top-k} omission probability}

\begin{lemma}[Top-$k$ omission bound]
\label{lem:kl-topk}
Fix a decision type $\tau$ and suppress $\tau$ in the notation.
Let $S\sim\calD_\tau$ and, conditionally on $S=s$, let $A\sim q(\cdot\mid s)$ be the teacher action used at that decision.
Assume the finite-candidate and margin conditions in the body:
\[
\PP(\Delta_k(S)\le u)\le C u^\beta
\qquad (u>0),
\]
where
\[
\Delta_k(s):=q_{(k)}(s)-q_{(k+1)}(s)
\]
is the gap between the $k$th and $(k+1)$th teacher probabilities in $\calA_\tau(s)$ after the fixed tie-breaking rule.
Then, for every student $\pi$ with finite $e_\tau(\pi)$,
\[
\PP\!\left(A\notin T_k(\pi(\cdot\mid S))\right)
\le
\eta_\tau
+
c(C,\beta)e_\tau(\pi)^{\beta/(\beta+2)},
\]
where
\[
c(C,\beta)
=
C^{2/(\beta+2)}
\left[
\left(\frac{4}{\beta}\right)^{\beta/(\beta+2)}
+
2\left(\frac{\beta}{4}\right)^{2/(\beta+2)}
\right].
\]
\end{lemma}

\begin{proof}
Fix a state $s$ and write
\[
Q=q(\cdot\mid s),\qquad P=\pi(\cdot\mid s),\qquad
T_Q=T_k(Q),\qquad T_P=T_k(P).
\]
The teacher action can be omitted from the student's top-$k$ list only in two ways:
either it is already outside the teacher's own top-$k$ set, or the teacher and student top-$k$ sets differ.
Thus
\[
\PP(A\notin T_P\mid S=s)
=Q(T_P^c)
\le
Q(T_Q^c)+\mathbf 1\{T_Q\ne T_P\}.
\]
The first term is exactly $\eta_{\tau,k}(s)$.

We now lower-bound the total-variation distance required for $T_Q\ne T_P$.
If $T_Q\ne T_P$, choose $a\in T_Q\setminus T_P$ and $b\in T_P\setminus T_Q$.
By the definition of the teacher top-$k$ set,
\[
Q(a)-Q(b)\ge \Delta_k(s).
\]
By the definition of the student top-$k$ set, $P(b)\ge P(a)$.
Therefore
\begin{align*}
\Delta_k(s)
&\le Q(a)-Q(b)\\
&=
[Q(a)-P(a)]+[P(a)-P(b)]+[P(b)-Q(b)]\\
&\le |Q(a)-P(a)|+|P(b)-Q(b)|.
\end{align*}
The right-hand side is bounded by the $\ell_1$ distance between $Q$ and $P$.
Consequently,
\[
T_Q\ne T_P
\quad\Longrightarrow\quad
\tv(Q,P)\ge \Delta_k(s)/2.
\]
For any $u>0$,
\[
\mathbf 1\{T_Q\ne T_P\}
\le
\mathbf 1\{\Delta_k(s)\le u\}
+
\mathbf 1\{\tv(Q,P)\ge u/2\}.
\]
Pinsker's inequality, $\tv(Q,P)^2\le \mathrm{KL}(Q\|P)/2$, implies
\[
\mathbf 1\{\tv(Q,P)\ge u/2\}
\le
\frac{2\,\mathrm{KL}(Q\|P)}{u^2}.
\]
Combining the displays gives
\[
\PP(A\notin T_P\mid S=s)
\le
\eta_{\tau,k}(s)
+
\mathbf 1\{\Delta_k(s)\le u\}
+
\frac{2\,\mathrm{KL}(q(\cdot\mid s)\|\pi(\cdot\mid s))}{u^2}.
\]
Taking expectation over $S\sim\calD_\tau$ and using the margin condition yields
\[
\PP(A\notin T_k(\pi(\cdot\mid S)))
\le
\eta_\tau+C u^\beta+\frac{2 e_\tau(\pi)}{u^2}.
\]
If $e_\tau(\pi)=0$, let $u\downarrow0$ and the last two terms vanish.
If $e_\tau(\pi)>0$, choose
\[
u=\left(\frac{4e_\tau(\pi)}{\beta C}\right)^{1/(\beta+2)}.
\]
Substitution gives the stated constant and exponent.
\end{proof}

\subsection{Setting 1: flat proving}

\begin{proof}[Proof of \cref{thm:flat-neurips}]
Draw a successful flat teacher trace from the success-conditioned law:
\[
Y_{\mathrm{flat}}=((S_1,A_1),\ldots,(S_L,A_L)).
\]
For each decision index $i$, let
\[
O_i:=\{A_i\notin T_k(\pi_{\mathrm{flat}}(\cdot\mid S_i))\}
\]
be the event that the required teacher action at that decision is omitted from the student's top-$k$ list.
By the search-completeness assumption, if none of the events $O_i$ occurs, flat top-$k$ search recovers a verified proof within the horizon.
Therefore
\[
1-V_T^{\pi_{\mathrm{flat}}}(s_0)
\le
\mathbf 1\left\{\bigcup_{i=1}^{L}O_i\right\}
\le
\sum_{i=1}^{L}\mathbf 1\{O_i\}.
\]
Taking expectation and using the occurrence-weighted definition of $\calD_{\mathrm{flat}}$,
\[
1-\EE_{s_0\sim q_{0\mid\esucc}}[V_T^{\pi_{\mathrm{flat}}}(s_0)]
\le
L_{\mathrm{flat}}\,
\PP_{(S,A)\sim\calD_{\mathrm{flat}}}\!\left(A\notin T_k(\pi_{\mathrm{flat}}(\cdot\mid S))\right).
\]
Applying \cref{lem:kl-topk} with $\tau=\mathrm{flat}$ gives
\[
1-\EE[V_T^{\pi_{\mathrm{flat}}}(s_0)]
\le
L_{\mathrm{flat}}
\Bigl(\eta_{\mathrm{flat}}
+c_{\mathrm{flat}}e_{\mathrm{flat}}(\pi_{\mathrm{flat}})^{p_{\mathrm{flat}}}\Bigr),
\]
which is the claimed bound after applying the Setting~1 shorthand
$L=L_{\mathrm{flat}}$, $\eta=\eta_{\mathrm{flat}}$, $e=e_{\mathrm{flat}}$, $p=p_{\mathrm{flat}}$, and $c=c_{\mathrm{flat}}$.
\end{proof}

\subsection{Setting 1: geometric ERM rate}

\begin{proof}[Proof of \cref{thm:flat-geom-neurips}]
Let $H=H_{\mathrm{flat}}$ and let
\[
\ell_\pi(s,a):=-\log \pi(a\mid s),
\qquad
\mathcal F:=\{\ell_\pi:\pi\in H\}.
\]
By \cref{ass:geometry}, there is a probability floor $p_{\min,\mathrm{flat}}>0$ on the support of the success-filtered flat decision process.
Thus $0\le \ell_\pi\le B_{\mathrm{flat}}:=\log(1/p_{\min,\mathrm{flat}})$ for every $\pi\in H$.
Let $R_{\mathrm{flat}}(\pi)$ be the population log-loss risk under the occurrence-weighted flat decision law:
\[
R_{\mathrm{flat}}(\pi)
:=
\EE_{(S,A)\sim\calD_{\mathrm{flat}}}
[-\log \pi(A\mid S)].
\]
Since $A\mid S=s$ is distributed as $q_{\mathrm{flat}}(\cdot\mid s)$,
\[
R_{\mathrm{flat}}(\pi)
=
\EE_{S\sim\calD_{\mathrm{flat}}}
\!\left[
H(q_{\mathrm{flat}}(\cdot\mid S))
+\mathrm{KL}\bigl(q_{\mathrm{flat}}(\cdot\mid S)\|\pi(\cdot\mid S)\bigr)
\right].
\]
The entropy term is independent of $\pi$, hence
\[
R_{\mathrm{flat}}(\pi)-\inf_{\pi'\in H}R_{\mathrm{flat}}(\pi')
=
e_{\mathrm{flat}}(\pi)-\inf_{\pi'\in H}e_{\mathrm{flat}}(\pi').
\]

Let $\hat R_{\mathrm{flat},N}$ be the empirical loss used by the ERM estimator, written in occurrence-weighted form.
By the random-denominator geometric uniform convergence bound in
\cref{prop:geom-random-denom}, with probability at least $1-\delta$,
\[
\sup_{\pi\in H}
\left|R_{\mathrm{flat}}(\pi)-\hat R_{\mathrm{flat},N}(\pi)\right|
\le
\varepsilon_N
:=
C_{\mathrm{flat}}\left(
\Psi_{d_{\mathrm{flat}}}(N)
+
\sqrt{\frac{\log(1/\delta)}{N}}
\right).
\]
Because $\hat\pi_N$ minimizes $\hat R_{\mathrm{flat},N}$ over $H$,
\begin{align*}
R_{\mathrm{flat}}(\hat\pi_N)
&\le
\hat R_{\mathrm{flat},N}(\hat\pi_N)
+
\sup_{\pi\in H}|R_{\mathrm{flat}}(\pi)-\hat R_{\mathrm{flat},N}(\pi)|\\
&\le
\hat R_{\mathrm{flat},N}(\pi)
+
\sup_{\pi\in H}|R_{\mathrm{flat}}(\pi)-\hat R_{\mathrm{flat},N}(\pi)|\\
&\le
R_{\mathrm{flat}}(\pi)
+
2\sup_{\pi'\in H}|R_{\mathrm{flat}}(\pi')-\hat R_{\mathrm{flat},N}(\pi')|
\end{align*}
for every $\pi\in H$.
Taking the infimum over $\pi\in H$ gives
\[
R_{\mathrm{flat}}(\hat\pi_N)
\le
\inf_{\pi\in H}R_{\mathrm{flat}}(\pi)
+
2\varepsilon_N.
\]
Subtracting the teacher entropy term converts the same inequality to $e_{\mathrm{flat}}$ and proves
\[
e_{\mathrm{flat}}(\hat\pi_N)
\le
\inf_{\pi\in H_{\mathrm{flat}}} e_{\mathrm{flat}}(\pi)
+
O\!\left(
\Psi_{d_{\mathrm{flat}}}(N)
+
\sqrt{\frac{\log(1/\delta)}{N}}
\right).
\]
\end{proof}

\subsection{Setting 2: hierarchical proving}

\begin{proof}[Proof of \cref{thm:hier-neurips}]
Let $\mathcal C_{\mathrm{graph}}$ be the graph-coverage event defined in Setting~2.
It fails with probability $\Delta_{\mathrm{graph}}$.

Let
\[
Y_{\mathrm{hi}}=((S^{\mathrm{hi}}_1,A^{\mathrm{hi}}_1),\ldots,
(S^{\mathrm{hi}}_{L_{\mathrm{hi}}'},A^{\mathrm{hi}}_{L_{\mathrm{hi}}'}))
\]
be the high-level decision trace of a successful hierarchical certificate, and let
\[
Y_{\mathrm{lo}}=((S^{\mathrm{lo}}_1,A^{\mathrm{lo}}_1),\ldots,
(S^{\mathrm{lo}}_{L_{\mathrm{lo}}'},A^{\mathrm{lo}}_{L_{\mathrm{lo}}'}))
\]
be the concatenation of the low-level local decisions used to solve the unique terminal blocks.
The random lengths have means $L_{\mathrm{hi}}$ and $L_{\mathrm{lo}}$.
Define omission events
\[
O^{\mathrm{hi}}_i
:=
\{A^{\mathrm{hi}}_i\notin T_k(\pi_{\mathrm{hi}}(\cdot\mid S^{\mathrm{hi}}_i))\},
\qquad
O^{\mathrm{lo}}_j
:=
\{A^{\mathrm{lo}}_j\notin T_k(\pi_{\mathrm{lo}}(\cdot\mid S^{\mathrm{lo}}_j))\}.
\]
By hierarchical search completeness with memoization and graph coverage, search succeeds whenever $\mathcal C_{\mathrm{graph}}$ occurs and none of these high- or low-level omission events occurs.
Hence
\[
1-V_T^{(\gpred,\pi_{\mathrm{hi}},\pi_{\mathrm{lo}})}(s_0)
\le
\mathbf 1\{\mathcal C_{\mathrm{graph}}^c\}
+
\sum_{i=1}^{L_{\mathrm{hi}}'}\mathbf 1\{O_i^{\mathrm{hi}}\}
+
\sum_{j=1}^{L_{\mathrm{lo}}'}\mathbf 1\{O_j^{\mathrm{lo}}\}.
\]
Taking expectation and using the occurrence-weighted mixtures $\calD_{\mathrm{hi}}$ and $\calD_{\mathrm{lo}}$ gives
\begin{align*}
1-\EE_{s_0\sim q_{0,\mathrm{hier}\mid\esucc}}[V_T^{(\gpred,\pi_{\mathrm{hi}},\pi_{\mathrm{lo}})}(s_0)]
&\le
\Delta_{\mathrm{graph}}
+
L_{\mathrm{hi}}\,
\PP_{\calD_{\mathrm{hi}}}
\!\left(A\notin T_k(\pi_{\mathrm{hi}}(\cdot\mid S))\right)\\
&\quad+
L_{\mathrm{lo}}\,
\PP_{\calD_{\mathrm{lo}}}
\!\left(A\notin T_k(\pi_{\mathrm{lo}}(\cdot\mid S))\right).
\end{align*}
Apply \cref{lem:kl-topk} once with $\tau=\mathrm{hi}$ and once with $\tau=\mathrm{lo}$.
This yields
\begin{align*}
1-\EE_{s_0\sim q_{0,\mathrm{hier}\mid\esucc}}[V_T^{(\gpred,\pi_{\mathrm{hi}},\pi_{\mathrm{lo}})}(s_0)]
&\le
\Delta_{\mathrm{graph}}
+
L_{\mathrm{hi}}\Bigl(\eta_{\mathrm{hi}}
+c_{\mathrm{hi}}e_{\mathrm{hi}}(\pi_{\mathrm{hi}})^{p_{\mathrm{hi}}}\Bigr)\\
&\quad+
L_{\mathrm{lo}}\Bigl(\eta_{\mathrm{lo}}
+c_{\mathrm{lo}}e_{\mathrm{lo}}(\pi_{\mathrm{lo}})^{p_{\mathrm{lo}}}\Bigr).
\end{align*}
Rearranging proves the theorem.
\end{proof}

\subsection{Setting 2: latent graph learning}

The main text separates the training-time posterior
$\nu_\phi:\calS_0\times\calY_{\mathrm{flat}}\leadsto\Latent$ from the
test-time graph predictor $\gpred:\calS_0\leadsto\Latent$.
The posterior can inspect the completed inlined trace; the graph predictor
cannot.

\begin{lemma}[Occurrence-weighted posterior mismatch]
\label{lem:posterior-mismatch-neurips}
Fix $\tau\in\{\mathrm{hi},\mathrm{lo}\}$.
Let $\bar{\calD}_\tau^\star$ be the occurrence-weighted law of
$\tau$-decision pairs $(S,A)$ induced by the true posterior
$\PP_{\mathrm{hier}}(M\mid s_0,y_{\mathrm{flat}})$, and let
$\bar{\calD}_\tau^\phi$ be the corresponding decision-pair law induced by
$\nu_\phi$.
Set
\[
\Delta_{\mathrm{post},\tau}^{\mathrm{occ}}(\phi)
:=
\mathrm{KL}(\bar{\calD}_\tau^\phi\|\bar{\calD}_\tau^\star).
\]
For $\pi\in H_\tau$, define
\[
R_\tau^\star(\pi)
:=
\EE_{(S,A)\sim\bar{\calD}_\tau^\star}\!\left[-\log\pi(A\mid S)\right],
\qquad
R_\tau^\phi(\pi)
:=
\EE_{(S,A)\sim\bar{\calD}_\tau^\phi}\!\left[-\log\pi(A\mid S)\right].
\]
If the policy class has probability floor $\rho_\tau$ and
$B_\tau=\log(1/\rho_\tau)$, then uniformly over $\pi\in H_\tau$,
\[
\left|R_\tau^\star(\pi)-R_\tau^\phi(\pi)\right|
\le
B_\tau\sqrt{\Delta_{\mathrm{post},\tau}^{\mathrm{occ}}(\phi)/2}.
\]
\end{lemma}

\begin{proof}
For a decision pair $(s,a)$, define
$\ell_\pi(s,a)=-\log\pi(a\mid s)$.
The probability floor gives $0\le \ell_\pi(s,a)\le B_\tau$.
Therefore
\[
\left|R_\tau^\star(\pi)-R_\tau^\phi(\pi)\right|
=
\left|\EE_{\bar{\calD}_\tau^\star}\ell_\pi
-\EE_{\bar{\calD}_\tau^\phi}\ell_\pi\right|
\le
B_\tau\,\tv(\bar{\calD}_\tau^\star,\bar{\calD}_\tau^\phi).
\]
Pinsker's inequality gives
$\tv(\bar{\calD}_\tau^\star,\bar{\calD}_\tau^\phi)
\le\sqrt{\mathrm{KL}(\bar{\calD}_\tau^\phi\|\bar{\calD}_\tau^\star)/2}$,
which proves the claim.
\end{proof}

\begin{proof}[Proof of \cref{thm:hier-geom-neurips}]
Condition on the fixed inferred posterior $\nu_\phi$; equivalently, use
sample splitting so that $\phi$ is independent of the policy ERM sample.
We prove the high-level bound; the low-level proof is identical.

Let $R_{\mathrm{hi}}^\phi(\pi)$ be the population posterior-weighted
occurrence log-loss induced by $\nu_\phi$, and let
$R_{\mathrm{hi}}^\star(\pi)$ be the true occurrence log-loss induced by
$\PP_{\mathrm{hier}}(M\mid s_0,y_{\mathrm{flat}})$.
Under the true decision-pair law,
\[
R_{\mathrm{hi}}^\star(\pi)
=
H_{\mathrm{hi}}^\star
+
e_{\mathrm{hi}}(\pi),
\]
where $H_{\mathrm{hi}}^\star$ is the conditional entropy of the true
high-level action law and is independent of $\pi$.
The random-denominator geometric uniform convergence bound in
\cref{prop:geom-random-denom}, applied conditionally on the fixed
posterior $\nu_\phi$, and the ERM inequality give, with probability at
least $1-\delta/2$,
\[
R_{\mathrm{hi}}^\phi(\hat\pi_N)
\le
\inf_{\pi\in H_{\mathrm{hi}}}R_{\mathrm{hi}}^\phi(\pi)
+
O\!\left(
\Psi_{d_{\mathrm{hi}}}(N)
+
\sqrt{\frac{\log(1/\delta)}{N}}
\right).
\]
By \cref{lem:posterior-mismatch-neurips}, uniformly over $\pi$,
\[
\left|R_{\mathrm{hi}}^\star(\pi)-R_{\mathrm{hi}}^\phi(\pi)\right|
\le
O\!\left(\sqrt{\Delta_{\mathrm{post},\mathrm{hi}}^{\mathrm{occ}}(\phi)}\right),
\]
and hence
\[
R_{\mathrm{hi}}^\star(\hat\pi_N)
\le
\inf_{\pi\in H_{\mathrm{hi}}}R_{\mathrm{hi}}^\star(\pi)
+
O\!\left(
\Psi_{d_{\mathrm{hi}}}(N)
+
\sqrt{\frac{\log(1/\delta)}{N}}
+
\sqrt{\Delta_{\mathrm{post},\mathrm{hi}}^{\mathrm{occ}}(\phi)}
\right).
\]
Subtracting the entropy term $H_{\mathrm{hi}}^\star$ gives the asserted
high-level bound for $e_{\mathrm{hi}}(\hat\pi_N)$.
Repeating the argument for $\tau=\mathrm{lo}$ and union-bounding the two
events proves the theorem.
\end{proof}

\begin{lemma}[Graph posterior mismatch]
\label{lem:graph-posterior-mismatch}
Let
$p^\star(m\mid s_0,y)$ be the true posterior over proof DAGs and let
$\nu_\phi(m\mid s_0,y)$ be the inferred posterior.
Define
\[
\Delta_{\mathrm{post}}(\phi)
:=
\EE_{(S_0,Y)}
\mathrm{KL}\!\left(
\nu_\phi(\cdot\mid S_0,Y)
\middle\|
p^\star(\cdot\mid S_0,Y)
\right).
\]
Here $(S_0,Y)$ follows the success-filtered marginal law of the observed
inlined trace.
If graph log-losses are bounded by $B_{\Latent}$, then uniformly over
$\gpred\in H_{\Latent}$,
\[
\left|R_{\Latent}(\gpred)-R_{\Latent}^{\phi}(\gpred)\right|
\le
B_{\Latent}\sqrt{\Delta_{\mathrm{post}}(\phi)/2},
\]
where $R_{\Latent}$ is the true graph log-loss risk under $(S_0,M)$ and
$R_{\Latent}^{\phi}$ is the soft-label risk induced by $\nu_\phi$.
\end{lemma}

\begin{proof}
If $\nu_\phi=p^\star$, then the soft labels average to the true conditional
law of $M$ given $S_0$ by the tower property.
For general $\nu_\phi$, apply the same bounded-loss and Pinsker argument
conditionally on $(s_0,y)$ and then Jensen's inequality over $(S_0,Y)$.
\end{proof}

\begin{proof}[Proof of \cref{thm:graph-pred-neurips}]
The posterior-weighted graph ERM objective estimates $R_{\Latent}^{\phi}$.
By \cref{ass:graph-learning}, the geometric uniform convergence argument of
\cref{prop:geom-random-denom} applies to this one-label-per-trace graph
objective, conditionally on the fixed posterior.  The standard ERM inequality
therefore gives
\[
R_{\Latent}^{\phi}(\hat{\gpred}_N)
\le
\inf_{\gpred\in H_{\Latent}}R_{\Latent}^{\phi}(\gpred)
+
O\!\left(
\Psi_{d_{\Latent}}(N)
+
\sqrt{\frac{\log(1/\delta)}{N}}
\right).
\]
\Cref{lem:graph-posterior-mismatch} transfers this inequality to the true
graph risk and adds the $\sqrt{\Delta_{\mathrm{post}}(\phi)}$ term.
This proves the displayed bound on $e_{\Latent}(\hat{\gpred}_N)$.
The displayed bound on $\Delta_{\mathrm{graph}}(\hat{\gpred}_N)$ is then
exactly the low-noise conversion in \cref{ass:graph-learning}.
\end{proof}

\begin{proof}[Proof of \cref{cor:observed-graph-main}]
When $M$ is observed, take
$\nu_\phi(\cdot\mid s_0,y_{\mathrm{flat}},M)=\delta_M$.
Then $\Delta_{\mathrm{post}}(\phi)=0$ and
$\Delta_{\mathrm{post},\tau}^{\mathrm{occ}}(\phi)=0$ for
$\tau\in\{\mathrm{hi},\mathrm{lo}\}$.
The graph objective becomes supervised graph log-loss, and the policy
objectives become supervised high- and low-level log-losses on the observed
decision traces.
The preceding proofs apply with all posterior-mismatch terms removed.
\end{proof}

\subsection{Setting 3: sufficient trace counts}

\begin{proof}[Proof of \cref{thm:sample-neurips}]
The quantities $\bar N_{\mathrm{flat}}(\delta)$ and
$\bar N_{\mathrm{hier}}(\delta)$ are certified sufficient thresholds obtained
from the upper bounds; they are not claimed to be minimal sample complexities.
By \cref{thm:flat-neurips},
\[
1-\EE_{s_0\sim q_{0\mid\esucc}}[V_T^{\hat\pi_N}(s_0)]
\le
L_{\mathrm{flat}}\eta_{\mathrm{flat}}
+
L_{\mathrm{flat}}c_{\mathrm{flat}}
e_{\mathrm{flat}}(\hat\pi_N)^{p_{\mathrm{flat}}}.
\]
The assumption $L_{\mathrm{flat}}\eta_{\mathrm{flat}}\le \delta/2$ leaves it enough to require
\[
L_{\mathrm{flat}}c_{\mathrm{flat}}
e_{\mathrm{flat}}(\hat\pi_N)^{p_{\mathrm{flat}}}
\le
\delta/2.
\]
If $e_{\mathrm{flat}}(\hat\pi_N)\le a_{\mathrm{flat}}N^{-\gamma_{\mathrm{flat}}}$, this is implied by
\[
N
\ge
\left(
\frac{2c_{\mathrm{flat}}a_{\mathrm{flat}}^{p_{\mathrm{flat}}}L_{\mathrm{flat}}}{\delta}
\right)^{1/(p_{\mathrm{flat}}\gamma_{\mathrm{flat}})}.
\]
Thus, ignoring integer ceilings, we take the flat certificate to be
\[
\bar N_{\mathrm{flat}}(\delta)
:=
\left(
\frac{2c_{\mathrm{flat}}a_{\mathrm{flat}}^{p_{\mathrm{flat}}}L_{\mathrm{flat}}}{\delta}
\right)^{1/(p_{\mathrm{flat}}\gamma_{\mathrm{flat}})},
\]
which gives the stated upper envelope up to constants and the logarithmic
factors hidden in the geometric learning curve.

For the hierarchical side, apply \cref{thm:hier-neurips} and use
\[
\Delta_{\mathrm{graph}}+L_{\mathrm{hi}}\eta_{\mathrm{hi}}+L_{\mathrm{lo}}\eta_{\mathrm{lo}}\le \delta/2.
\]
It suffices to require each remaining learning term to be at most
$\delta/4$:
\[
L_\tau c_\tau e_\tau(\hat\pi_N)^{p_\tau}\le \delta/4,
\qquad \tau\in\{\mathrm{hi},\mathrm{lo}\}.
\]
Using $e_\tau(\hat\pi_N)\le a_\tau N^{-\gamma_\tau}$, we take the certified
hierarchical threshold to be
\[
\bar N_{\mathrm{hier}}(\delta)
:=
\max_{\tau\in\{\mathrm{hi},\mathrm{lo}\}}
\left(
\frac{4c_\tau a_\tau^{p_\tau}L_\tau}{\delta}
\right)^{1/(p_\tau\gamma_\tau)}.
\]
This gives the displayed
$\max_{\tau}(L_\tau/\delta)^{1/(p_\tau\gamma_\tau)}$ upper envelope.
If all three learning terms share a common effective exponent
$p_{\mathrm{flat}}\gamma_{\mathrm{flat}}
=p_{\mathrm{hi}}\gamma_{\mathrm{hi}}
=p_{\mathrm{lo}}\gamma_{\mathrm{lo}}=:p\gamma$ and comparable constants, then
the explicit flat certificate divided by the explicit hierarchical certificate
obeys, up to constants,
\[
\bar N_{\mathrm{flat}}(\delta)/\bar N_{\mathrm{hier}}(\delta)
\gtrsim
\left(L_{\mathrm{flat}}/L_{\mathrm{hier}}\right)^{1/(p\gamma)}.
\]
Here we used
$\max\{L_{\mathrm{hi}},L_{\mathrm{lo}}\}\le L_{\mathrm{hier}}$.
For fixed $p\gamma>0$, in a theorem family indexed by a size/depth parameter $D$, substituting
$L_{\mathrm{flat}}/L_{\mathrm{hier}}=\exp(\Omega(D))$ into this display gives
\[
\bar N_{\mathrm{flat}}(\delta)/\bar N_{\mathrm{hier}}(\delta)\ge \exp(\Omega(D)).
\]
\end{proof}

\section{Sequential Generalization and Geometric Rates}
\label{sec:app-seq}

This section records the learning-theoretic input behind the body-level rate theorems.
The setup is standard: learn a policy by ERM under log-loss on sequentially generated decision points, control the resulting trace-level empirical process by sequential Rademacher complexity, and then bound that complexity by a doubling-dimension covering estimate \citep{Rakhlin2015sequential-lln,Rakhlin2015sequential-online,block2021majorization-sequential}.

\subsection{Random-denominator uniform convergence}

Fix one decision type, either the flat decisions or one of the hierarchical decision types.
Let $H$ be a policy class and let
\[
\ell_\pi(s,a):=-\log \pi(a\mid s)
\]
be the log-loss.
Assume a probability floor $\pi(a\mid s)\ge \rho>0$ on the relevant support, so $\ell_\pi\in[0,B]$ with $B=\log(1/\rho)$.
Let $\mathcal F:=\{\ell_\pi:\pi\in H\}$.
For a trace $Z=((S_1,A_1),\ldots,(S_L,A_L))$, write
\[
A_\pi(Z):=\sum_{i=1}^{L}\ell_\pi(S_i,A_i),
\qquad
\mu:=\EE[L],
\qquad
R(\pi):=\frac{\EE[A_\pi(Z)]}{\mu}.
\]
The empirical occurrence-weighted risk is
\[
\hat R_N(\pi)
:=
\frac{N^{-1}\sum_{n=1}^N A_\pi(Z_n)}
{N^{-1}\sum_{n=1}^N L_n}.
\]

\begin{proposition}[Geometric random-denominator uniform convergence]
\label{prop:geom-random-denom}
Assume \cref{ass:geometry} for the decision type under consideration.
Then there is a constant $C$ depending on the geometric constants in
\cref{ass:geometry}, but not on $N$, such that with probability at least
$1-\delta$,
\[
\sup_{\pi\in H}
\left|\hat R_N(\pi)-R(\pi)\right|
\le
C\left(
\Psi_d(N)+\sqrt{\frac{\log(1/\delta)}{N}}
\right),
\]
where $d$ is the doubling dimension of the represented state space and
\[
\Psi_d(N)\simeq
\begin{cases}
N^{-1/2}, & d=1,\\
N^{-1/2}\log N, & d=2,\\
N^{-1/d}, & d>2
\end{cases}
\]
up to polylogarithmic factors.
The same conclusion holds for posterior-weighted hierarchical objectives
after replacing $A_\pi(Z)$ and $L(Z)$ by their posterior expectations,
provided the posterior is fixed independently of the ERM sample.
\end{proposition}

\begin{proof}
Write $\bar A_N(\pi)=N^{-1}\sum_n A_\pi(Z_n)$ and
$\bar L_N=N^{-1}\sum_n L_n$.
By \cref{ass:geometry}, $0\le L\le L_{\max}$ and
$\mu=\EE[L]\ge\mu_{\min}>0$.
Hoeffding's inequality gives $\bar L_N\ge \mu/2$ and
$|\bar L_N-\mu|=O(L_{\max}\sqrt{\log(1/\delta)/N})$ with probability at
least $1-\delta/3$.
On this event,
\[
\sup_{\pi\in H}|\hat R_N(\pi)-R(\pi)|
\le
\frac{2}{\mu}\sup_{\pi\in H}
\left|\bar A_N(\pi)-\EE A_\pi(Z)\right|
+
\frac{2B}{\mu}\left|\bar L_N-\mu\right|.
\]
The second term is
$O((B L_{\max}/\mu_{\min})\sqrt{\log(1/\delta)/N})$.
For the first term, standard sequential symmetrization for $N$ independent
successful traces gives
\[
\sup_{\pi\in H}
\left|\bar A_N(\pi)-\EE A_\pi(Z)\right|
\le
O\!\left(
\mathfrak R_N^{\mathrm{tr}}(\mathcal F)
+
B L_{\max}\sqrt{\frac{\log(1/\delta)}{N}}
\right)
\]
with probability at least $1-\delta/3$, where
$\mathfrak R_N^{\mathrm{tr}}(\mathcal F)$ is the sequential Rademacher
complexity of the trace-level class $Z\mapsto A_\pi(Z)$.

It remains to bound this complexity.
For reference, the usual depth-$T$ sequential Rademacher complexity is
\[
\mathfrak R_T^{\mathrm{seq}}(\mathcal F)
:=
\sup_{\mathbf s}\,
\EE_\varepsilon\Big[\sup_{f\in\mathcal F}\frac1T\sum_{t=1}^T
\varepsilon_t f(\mathbf s_t(\varepsilon_{1:t-1}))\Big].
\]
Let $\mathcal N_{\mathrm{seq}}(\eps,\mathcal F,T)$ be the sequential covering number of $\mathcal F$ at scale $\eps$ and depth $T$.
A Dudley-type chaining bound gives
\[
\mathfrak R_T^{\mathrm{seq}}(\mathcal F)
\le
c\inf_{\alpha\in(0,B]}
\left[
\alpha+\frac{1}{\sqrt T}\int_\alpha^B
\sqrt{\log \mathcal N_{\mathrm{seq}}(\eps,\mathcal F,T)}\,d\eps
\right]
\]
for a universal constant $c$.
The probability floor and total-variation Lipschitzness imply that the
log-loss class is $L_\ell$-Lipschitz in the state variable with
$L_\ell\lesssim L_\pi/\rho$.
Since the represented state space is compact with diameter $D$ and doubling
dimension $d$, a net on the state space and a uniform quantization of the
finite action probabilities give
\[
\log \mathcal N(\eps,\mathcal F,\|\cdot\|_\infty)
\lesssim
\left(\frac{L_\ell D}{\eps}\right)^d
\log\!\left(\frac{B}{\eps}\right),
\]
and hence the same upper bound for sequential covering.
For trace sums, one applies this cover at scale $\eps/L_{\max}$, which only
changes constants and logarithmic factors.
Plugging the resulting entropy bound into the Dudley integral yields
\[
\mathfrak R_N^{\mathrm{tr}}(\mathcal F)
\lesssim
\begin{cases}
N^{-1/2}, & d=1,\\
N^{-1/2}\log N, & d=2,\\
N^{-1/d}, & d>2,
\end{cases}
\]
up to polylogarithmic factors and constants depending on
$B,L_{\max},\mu_{\min},D,L_\pi,\rho$.
Combining the numerator and denominator bounds and union-bounding the events
proves the display.
For posterior-weighted hierarchical objectives with fixed $\nu_\phi$, the
same argument applies to the bounded functions
$\EE_{M\sim\nu_\phi}[A_\pi(Z,M)]$ and
$\EE_{M\sim\nu_\phi}[L(Z,M)]$; sample splitting makes this conditioning
legitimate when $\nu_\phi$ is learned.
\end{proof}

The posterior-mismatch term for latent proof DAGs is part of the proof of \cref{thm:hier-geom-neurips} in \cref{sec:app-proof-sketches}, because it is specific to \cref{thm:hier-geom-neurips} rather than to sequential Rademacher complexity itself.

\section{When Matching Sample Orders Can Be Claimed}
\label{sec:app-matching-orders}

\cref{thm:sample-neurips} deliberately compares the certified sufficient
thresholds obtained by solving the upper bounds in
\cref{thm:flat-neurips,thm:hier-neurips}.  This is weaker than a statement about
the unknown minimal number of traces needed by a learning algorithm.  A matching
order statement is possible only after adding lower-bound assumptions that make
the upper-bound analysis tight.

For a fixed training procedure and confidence level $\alpha\in(0,1)$, define
\begin{align*}
N^\star_{\mathrm{flat}}(\delta;\alpha)
&:=
\inf\{N:\PP_{\mathrm{train}}(1-\mathrm{SP}_{T,\mathrm{flat}}(\hat\pi_N)\le\delta)\ge 1-\alpha\},\\
N^\star_{\mathrm{hier}}(\delta;\alpha)
&:=
\inf\{N:\PP_{\mathrm{train}}(1-\mathrm{SP}_{T,\mathrm{hier}}(\hat{\gpred}_N,\hat\pi_{\mathrm{hi},N},\hat\pi_{\mathrm{lo},N})\le\delta)\ge 1-\alpha\}.
\end{align*}
We suppress $\alpha$ below.
To obtain matching orders for these actual thresholds, the following conditions
are sufficient in the small-error regime.

\paragraph{Two-sided learning curves.}
For each relevant decision type
$\tau\in\{\mathrm{flat},\mathrm{hi},\mathrm{lo}\}$, the learned policy should
satisfy
\[
\underline a_\tau N^{-\gamma_\tau}
\lesssim
e_\tau(\hat\pi_{\tau,N})
\lesssim
\overline a_\tau N^{-\gamma_\tau}
\]
up to logarithmic factors and approximation floors.  The main text assumes only
the upper inequality because it is enough for a sufficient-sample certificate.
The lower inequality is a genuine hardness assumption on the chosen data
distribution, hypothesis class, and learning algorithm.  If graph prediction is
not negligible, an analogous two-sided curve is needed for the graph coverage
term, or else the graph threshold must be included in the hierarchical maximum.

\paragraph{Two-sided conversion from log-loss to top-$k$ omissions.}
Let
\[
P_\tau(\pi)
:=
\PP\!\left(A_\tau\notin T_k(\pi(\cdot\mid S_\tau))\right)
\]
be the probability that the teacher action at a success-conditioned decision
occurrence of type $\tau$ is not included in the student's top-$k$ list.
\cref{lem:kl-topk} gives the upper bound
$P_\tau(\pi)\le \eta_\tau+c_\tau e_\tau(\pi)^{p_\tau}$.
For matching orders one also needs a reverse inequality of the form
\[
\eta_\tau+\underline b_\tau e_\tau(\pi)^{p_\tau}
\lesssim
P_\tau(\pi)
\lesssim
\eta_\tau+\overline b_\tau e_\tau(\pi)^{p_\tau}
\]
for the learned policies in the regime being compared.  This requires more than
the Tsybakov margin used for the upper bound: the states near the top-$k$
boundary must have enough mass, and the learner's residual errors must affect
the ranking-relevant directions rather than only probabilities that do not
change the top-$k$ set.

\paragraph{No-rescue lower bound for search.}
The success guarantees use a one-way implication: if all required teacher
actions appear in the test-time top-$k$ lists, the search succeeds.  To lower
bound the necessary sample count, failures of required local decisions must also
cause proof failure at the same order.  In the target range this can be expressed
as
\[
1-\mathrm{SP}_{T,\mathrm{flat}}(\hat\pi_N)
\asymp
L_{\mathrm{flat}}P_{\mathrm{flat}}(\hat\pi_N)
\]
and
\[
1-\mathrm{SP}_{T,\mathrm{hier}}(\hat{\gpred}_N,\hat\pi_{\mathrm{hi},N},\hat\pi_{\mathrm{lo},N})
\asymp
\Delta_{\mathrm{graph}}(\hat{\gpred}_N)
+L_{\mathrm{hi}}P_{\mathrm{hi}}(\hat\pi_{\mathrm{hi},N})
+L_{\mathrm{lo}}P_{\mathrm{lo}}(\hat\pi_{\mathrm{lo},N}).
\]
This is a structural assumption on the proof family and the search procedure.
It rules out frequent rescue by alternate proof paths, by unusually strong
test-time search, or by a flat representation that perfectly canonicalizes and
reuses the repeated subproblem.

\paragraph{Floors below the target accuracy.}
The irreducible terms must not dominate the comparison.  For a fixed target
range it is enough that
$L_{\mathrm{flat}}\eta_{\mathrm{flat}}$ and
$\Delta_{\mathrm{graph}}+L_{\mathrm{hi}}\eta_{\mathrm{hi}}
 +L_{\mathrm{lo}}\eta_{\mathrm{lo}}$
are bounded by a fixed fraction of $\delta$.  For an asymptotic
$\delta\downarrow0$ statement, these floors must vanish or be driven below
$o(\delta)$ by the chosen search budget and representation.

Under these additional conditions, and when the graph term is negligible or has
already been included in the hierarchical maximum, the actual trace thresholds
obey
\[
N^\star_{\mathrm{flat}}(\delta)
\asymp
\left(L_{\mathrm{flat}}/\delta\right)^{1/(p_{\mathrm{flat}}\gamma_{\mathrm{flat}})},
\qquad
N^\star_{\mathrm{hier}}(\delta)
\asymp
\max_{\tau\in\{\mathrm{hi},\mathrm{lo}\}}
\left(L_\tau/\delta\right)^{1/(p_\tau\gamma_\tau)}.
\]
In the common-exponent case
$p_{\mathrm{flat}}\gamma_{\mathrm{flat}}
=p_{\mathrm{hi}}\gamma_{\mathrm{hi}}
=p_{\mathrm{lo}}\gamma_{\mathrm{lo}}=:p\gamma$ with comparable constants,
$\max\{L_{\mathrm{hi}},L_{\mathrm{lo}}\}\asymp L_{\mathrm{hi}}+L_{\mathrm{lo}}$
up to a factor of two, and hence
\[
N^\star_{\mathrm{flat}}(\delta)/N^\star_{\mathrm{hier}}(\delta)
\asymp
\left(L_{\mathrm{flat}}/L_{\mathrm{hier}}\right)^{1/(p\gamma)}.
\]

Thus a matching-order separation is not a consequence of the one-sided
certificates alone.  It is a stronger statement for proof families where
unfolded repeated decisions are statistically hard, ranking errors translate
back into top-$k$ omissions, and test-time search cannot usually repair a missed
required action.

\section{Example: Cut-Elimination Blows Up Proof Size Exponentially} \label{sec:example-cut}

This section gives the proof-theoretic example behind the length separation
used in the main text.
The point is simple: a proof may first derive a reusable intermediate formula
$Y$ and then use $Y$ many times, whereas the cut-free version must inline the
derivation of $Y$ at each use site.
When the use sites form the leaves of a binary proof tree, this inlining can
duplicate the same sub-derivation $2^d$ times.
The following \cref{ex:two-m} translates this proof-size phenomenon
into the decision counts $L_{\mathrm{hi}},L_{\mathrm{lo}}$, and
$L_{\mathrm{flat}}$ used in Setting~3.

In sequent calculus, a \emph{cut} is the rule
\[
\infer[\mathrm{cut}]{\Gamma,\Delta\vdash C}{
  \Gamma\vdash Y
  &
  \Delta,Y\vdash C
}.
\]
The intermediate formula $Y$ is proved once and then consumed by the rest of
the derivation.

\emph{Cut-introduction} means deliberately adding such an intermediate lemma or
proof block, even when it is not part of the final statement.

In proof assistants such as Lean, this is analogous to introducing a named local fact such
as \lstinline{have Y : B := ...} and reusing it later.

\emph{Cut-elimination} removes these intermediate cuts while preserving
derivability.
The transformation is logically sound, but it may destroy sharing by replacing
each use of $Y$ with a copy of its proof.

\citet{boolos1984dont-eliminate-cut} emphasized that eliminating cuts can make
proofs exponentially longer.
The example below follows the same mechanism as a toy instance. %
We build a long derivation of $B$ from $A$ and then require $B$ at several leaf
goals.
With a cut, the long derivation appears once; after cut elimination, it is
copied at every leaf.

\subsection{Informal Overview}

Consider proving the following Theorem~$X$ in two ways with and without reusing intermediate Lemma~$Y$.

\begin{namedthm*}[$X$ : ``$A$ ..., then $B$ and $B$, and $B$ and $B$'']\label{thm:X}
Assume $A$ and the following implication chain
\begin{align}
A\Rightarrow B_1, \
B_1\Rightarrow B_2, \
B_2\Rightarrow B_3, \
B_3\Rightarrow B_4, \
B_4\Rightarrow B_5, \
B_5\Rightarrow B_6, \text{ and }
B_6\Rightarrow B. \tag{$\star$} \label{eq:AtoB}
\end{align}
Then we have
\begin{align}
  (B \wedge B) \wedge (B \wedge B). \tag{$\star\star$} \label{eq:BBBB}
\end{align}
\end{namedthm*}

\begin{namedlem*}[$Y$ : ``we have $B$ because $A$'']\label{lem:Y}
Assume $A$ and (\ref{eq:AtoB}), then $B$.
\end{namedlem*}

In fact, deriving $B$ from $A$ is quite long as follows. Thus it is reasonable to name the proof $Y$ and reuse it.

\begin{proof}[Proof of $Y$]
We have $B_1$ because we assume $A$ and $A \Rightarrow B_1$. %
We have $B_2$ because we have $B_1$ and $B_1 \Rightarrow B_2$. %
We have $B_3$ because we have $B_2$ and $B_2 \Rightarrow B_3$. %
We have $B_4$ because we have $B_3$ and $B_3 \Rightarrow B_4$. %
We have $B_5$ because we have $B_4$ and $B_4 \Rightarrow B_5$. %
We have $B_6$ because we have $B_5$ and $B_5 \Rightarrow B_6$. %
Therefore, we have $B$ because we have $B_6$ and $B_6 \Rightarrow B$. %
\end{proof}

The \emph{cut-aware} proof of $X$ derives \(Y\) once and reuses it to close the $2^2=4$
occurrences of \(B\).

\begin{proof}[Proof of $X$ reusing $Y$]
    We have $B$ because we have $A$ and $Y : A \Rightarrow B$,
    so we have $B \wedge B$ because we have $B$,
    therefore we have $(B \wedge B) \wedge (B \wedge B)$ because we have $B \wedge B$.
\end{proof}

On the other hand, the \emph{cut-free} proof of $X$ has no named \(Y\), so the same chain derivation must be
repeated at those occurrences.

\begin{proof}[Proof of $X$ without reusing $Y$]
We have $B \wedge B$ because (the 1st copy of proof $Y$) and (the 2nd copy of proof $Y$).
We have $B \wedge B$ because (the 3rd copy of proof $Y$) and (the 4th copy of proof $Y$).
Therefore we have $(B \wedge B) \wedge (B \wedge B)$ because we have $B \wedge B$ and $B \wedge B$.
\end{proof}

In other words, the cut-free proof is $2^2=4$ times as long as the cut-aware proof.
While it is easy to identify higher-order structures such as $Y$ in this artificial example, identifying them in real-world proofs is not always straightforward. It may be more efficient to simply iterate through the proof step by step rather than trying to identify higher-order structures. The motivation for this study is to analyze this dilemma.

\subsection{Proof explosion in sequent calculus} %

Let the context be \[
\Gamma = \{A, A\to B_1, B_1\to B_2, B_2\to B_3, B_3\to B_4, B_4\to B_5, B_5\to B_6, B_6\to B\}.
\]
We derive the target
\[\pi_X : \Gamma \vdash (B\wedge B)\wedge(B\wedge B).\]

\subsubsection{\texorpdfstring{A long derivation
\(\pi_Y : \Gamma \vdash B\)}{A long derivation \textbackslash pi\_B : \textbackslash Gamma \textbackslash vdash B}}

Using a standard sequent calculus for \(\to\) and \(\wedge\), keep the
context fixed as the full set \(\Gamma\) above.  Weakening and contraction are
admissible, so every intermediate sequent below has the same left context
\(\Gamma\).  Write
\[
\pi_0:\Gamma\vdash A,\qquad
\pi_i:\Gamma\vdash B_i\ (i=1,\ldots,6),\qquad
\pi_7:\Gamma\vdash B .
\]
The chain is obtained by repeated implication elimination in sequent form:
\[
\infer[\to L]{\pi_i:\Gamma\vdash B_i}{
  \pi_{i-1}:\Gamma\vdash B_{i-1}
  &
  \infer[ax]{\Gamma,B_i\vdash B_i}{}
}
\qquad (i=1,\ldots,6),
\]
where \(B_0:=A\) and the formula \(B_{i-1}\to B_i\) is an element of
\(\Gamma\).  The final step is
\[
\infer[\to L]{\pi_7:\Gamma\vdash B}{
  \pi_6:\Gamma\vdash B_6
  &
  \infer[ax]{\Gamma,B\vdash B}{}
},
\]
using \(B_6\to B\in\Gamma\).

This chain has length proportional to the number of implication-links
(here, 7). Denote the whole derivation by \(\pi_Y := \pi_7\).

\subsubsection{\texorpdfstring{Cut-aware (derive \(B\) once, then reuse it)}{Cut-aware (derive B once, then reuse it)}}

\[
\infer[\mathrm{cut}]{\Gamma \vdash (B\wedge B)\wedge(B\wedge B)}{
\pi_Y:\ \Gamma \vdash B
&
\infer[\wedge R]{\Gamma, B \vdash (B\wedge B)\wedge(B\wedge B)}{
\infer[\wedge R]{\Gamma, B \vdash B\wedge B}{
\infer[\mathrm{ax}]{\Gamma,B \vdash B}{} &
\infer[\mathrm{ax}]{\Gamma,B \vdash B}{}
}
&
\infer[\wedge R]{\Gamma, B \vdash B\wedge B}{
\infer[\mathrm{ax}]{\Gamma,B \vdash B}{} &
\infer[\mathrm{ax}]{\Gamma,B \vdash B}{}
}
}
}
\]

Here the expensive part \(\pi_Y\) appears \textbf{once}; everything else
is \(O(1)\) (just \(\mathrm{ax}\) and \(\wedge R\)).

\subsubsection{\texorpdfstring{Cut-free at the outer level (must supply \(2^2=4\) copies of \(\pi_Y\))}{Cut-free at the outer level (must supply two-to-two equal four copies of piY)}}

\[
\infer[\wedge R]{\Gamma \vdash (B\wedge B)\wedge(B\wedge B)}{
\infer[\wedge R]{\Gamma \vdash B\wedge B}{
\pi_Y:\ \Gamma \vdash B
&
\pi_Y:\ \Gamma \vdash B
}
&
\infer[\wedge R]{\Gamma \vdash B\wedge B}{
\pi_Y:\ \Gamma \vdash B
&
\pi_Y:\ \Gamma \vdash B
}
}
\]

Now the long derivation \(\pi_Y\) appears \textbf{$2^2$ times}. If
\(\pi_Y\) has length \(m\), the cut-free proof has length roughly
\(2^2 m + O(1)\), whereas the cut-aware proof has length \(m + O(1)\).

\textbf{Nested branching.}
For the recursive family
\[
B^{(0)} := B,\qquad B^{(d+1)} := B^{(d)} \wedge B^{(d)},
\]
the same mechanism creates \(2^d\) leaves.
Cut elimination then duplicates the derivation of \(B\) at those leaves, while
a proof with cut stores it once and reuses it.

\subsection{Proof explosion in Lean}
In Lean, we use the \lstinline{have} tactic for an explicit cut.

\newpage

\subsubsection{\texorpdfstring{Cut-aware (derive \(B\) once, then reuse it)}{Cut-aware (derive B once, then reuse it)}}

\begin{lstlisting}
theorem X_with_cut_aware_proof
    (A B1 B2 B3 B4 B5 B6 B : Prop)
    (hA : A)
    (f1 : A → B1)
    (f2 : B1 → B2)
    (f3 : B2 → B3)
    (f4 : B3 → B4)
    (f5 : B4 → B5)
    (f6 : B5 → B6)
    (f7 : B6 → B) :
    (B ∧ B) ∧ (B ∧ B) := by

    -- CUT: a shared intermediate lemma proved once
    have Y_we_have_B_because_A : B := f7 (f6 (f5 (f4 (f3 (f2 (f1 hA))))))

    have Z_we_have_B_and_B : B ∧ B := ⟨Y_we_have_B_because_A, Y_we_have_B_because_A⟩

    -- Use the lemmas without recomputing the long chain
    exact ⟨Z_we_have_B_and_B, Z_we_have_B_and_B⟩
\end{lstlisting}

\subsubsection{\texorpdfstring{Cut-free at the outer level (the same
\(B\)-derivation is duplicated \(2^2=4\)
times)}{Cut-free at the outer level (the same B-derivation is duplicated two-to-two equal four times)}}

\begin{lstlisting}
theorem X_with_cut_free_proof
    (A B1 B2 B3 B4 B5 B6 B : Prop)
    (hA : A)
    (f1 : A → B1)
    (f2 : B1 → B2)
    (f3 : B2 → B3)
    (f4 : B3 → B4)
    (f5 : B4 → B5)
    (f6 : B5 → B6)
    (f7 : B6 → B) :
    (B ∧ B) ∧ (B ∧ B) := by

    -- Recomputing the chain 2^2 times
    exact
     ⟨⟨f7 (f6 (f5 (f4 (f3 (f2 (f1 hA)))))),   -- first copy of the long derivation
      f7 (f6 (f5 (f4 (f3 (f2 (f1 hA))))))⟩,  -- second copy
     ⟨f7 (f6 (f5 (f4 (f3 (f2 (f1 hA)))))),   -- third copy
      f7 (f6 (f5 (f4 (f3 (f2 (f1 hA))))))⟩⟩  -- fourth copy
\end{lstlisting}

\textbf{Lean reading.}
The commented code mirrors the sequent-calculus construction.
The proof with \lstinline{have} names the intermediate proof of \(B\) and reuses
it; the cut-free proof inlines the same chain at each leaf.
Thus the proof-term size follows the same accounting: roughly \(m+O(1)\) with
the local lemma, and \(2^d m+O(1)\) after inlining in a depth-\(d\) binary
pattern.

\section{Details on High-Level Structures}%
\label{sec:app-hilevel-formal}

This appendix makes explicit the nonstandard high-level layer used in Settings~2--3.
The underlying proof system is still the ordinary verifier MDP
\[
\mathsf M_{\psystem}
=
(\calS,\calA_{\mathrm{lo}},F_{\mathrm{lo}},G,\iota)
\]
from \cref{sec:proof-mdp}.
The high-level layer is an additional representation of reusable proof structure; it is not an additional logical axiom system.

\subsection{Formal objects}

The following definitions fix the typed objects used by the high-level learner:
block interfaces, proof DAGs, high-level search states, and the unfolding map
back to an ordinary low-level proof trace.

\begin{definition}[Block interfaces]
\label{def:app-block-interface}
Let
\[
\mathsf{Judg}_{\psystem}
:=
\{(\Gamma,\varphi):\Gamma\in\mathcal P_{\mathrm{fin}}(\flang),\ \varphi\in\flang\}
\]
be the set of finite-premise proof obligations.
A block interface is a finite set of such obligations:
\[
\mathsf{Obl}_{\psystem}
\subseteq
\mathcal P_{\mathrm{fin}}(\mathsf{Judg}_{\psystem}).
\]
Equivalently, after choosing a concrete encoding of sequents or Lean goals as formulas, one may regard a block interface as an element of
\[
\mathcal P_{\mathrm{fin}}(\flang).
\]
This is the sense in which a labeled high-level node may be written as a finite set of formulas.
For each interface $I\in\mathsf{Obl}_{\psystem}$, let
\[
\jmath(I)\in\calS
\]
be the corresponding ordinary low-level initial proof state.
For example, if
$I=\{(\Gamma_1,\varphi_1),\ldots,(\Gamma_m,\varphi_m)\}$,
then $\jmath(I)$ is the proof state with open goals
$\Gamma_i\vdash_{\psystem}\varphi_i$ for $i=1,\ldots,m$.
\end{definition}

\begin{definition}[High-level proof DAG]
\label{def:app-hilevel-dag}
A high-level proof DAG is a tuple
\[
M=(V,E,r,\lambda),
\]
where $V$ is a finite set of block identifiers, $E\subseteq V\times V$ is acyclic, $r\in V$ is the root block, and
\[
\lambda:V\to\mathsf{Obl}_{\psystem}
\]
assigns a block interface to each node.
The edge convention is
\[
v\to w
\quad\Longleftrightarrow\quad
\text{the proof of block }v\text{ may call or consume block }w.
\]
Thus descendants are dependencies.
If $\lambda$ is injective and the block labels themselves are used as node identities, then the same object can be identified with
\[
V\subseteq \mathsf{Obl}_{\psystem}
\subseteq \mathcal P_{\mathrm{fin}}(\mathsf{Judg}_{\psystem}),
\]
or, under a formula-level encoding of judgements, with
$V\subseteq\mathcal P_{\mathrm{fin}}(\flang)$.
We keep the explicit identifier set $V$ because two syntactically equal interfaces may occur as distinct blocks before the learner decides whether to reuse them.
Let $\Latent$ denote the class of finite high-level proof DAGs satisfying this definition.
\end{definition}

\begin{definition}[High-level states and graph actions]
\label{def:app-hilevel-actions}
A high-level search state is a finite record
\[
s_{\mathrm{hi}}
=
(M_t,U_t,R_t,C_t),
\]
where $M_t=(V_t,E_t,r_t,\lambda_t)$ is the current partial high-level DAG, $U_t\subseteq V_t$ is the set of unresolved blocks, $R_t\subseteq V_t$ is the set of solved blocks, and $C_t$ stores the local certificates already attached to solved blocks.
The high-level action set is state dependent; the global alphabet is the disjoint union
\[
\calA_{\mathrm{hi}}
=
\calA_{\mathrm{sel}}
\sqcup
\calA_{\mathrm{exp}}
\sqcup
\calA_{\mathrm{ord}}
\sqcup
\calA_{\mathrm{reuse}}
\sqcup
\calA_{\mathrm{close}}.
\]
At a state $s_{\mathrm{hi}}=(M_t,U_t,R_t,C_t)$ these components have the following meanings.
\[
\calA_{\mathrm{sel}}(s_{\mathrm{hi}})
=
\{\operatorname{select}(v):v\in U_t\}
\]
chooses the unresolved block to work on.
\[
\calA_{\mathrm{exp}}(s_{\mathrm{hi}})
=
\{\operatorname{expand}(v;I_1,\ldots,I_m):
v\in U_t,\ I_j\in\mathsf{Obl}_{\psystem},\ m\ge0\}
\]
adds fresh child blocks with interfaces $I_1,\ldots,I_m$ and edges from $v$ to those children.
\[
\calA_{\mathrm{ord}}(s_{\mathrm{hi}})
\]
contains finite scheduling choices, such as a permutation or priority order of the currently available unresolved blocks, subject to the partial order induced by $E_t$.
\[
\calA_{\mathrm{reuse}}(s_{\mathrm{hi}})
=
\{\operatorname{reuse}(v,w):v\in U_t,\ w\in R_t,\ \lambda_t(w)\simeq\lambda_t(v)\}
\]
declares that the current block $v$ is discharged by an already solved compatible block $w$.
Here $\simeq$ is the chosen interface-equivalence relation, for example syntactic equality of normalized goals or alpha-equivalence of Lean local contexts.
Finally,
\[
\calA_{\mathrm{close}}(s_{\mathrm{hi}})
=
\left\{\operatorname{close}(v,c):
v\in U_t,\ 
c\in
\left(\calA_{\mathrm{lo}}
\sqcup
\{\operatorname{call}(w):(v,w)\in E_t,\ w\in R_t\}\right)^\ast
\right\}
\]
attaches a local certificate $c$ that verifies the interface $\lambda_t(v)$ using ordinary low-level actions and allowed calls to already solved child blocks.
These actions manipulate the proof DAG; they are distinct from low-level tactic actions in $\calA_{\mathrm{lo}}$.
\end{definition}

\begin{definition}[Local block certificates and unfolding]
\label{def:app-unfolding}
For a fixed high-level DAG $M=(V,E,r,\lambda)$, write
\[
\operatorname{Child}_M(v):=\{w\in V:(v,w)\in E\}.
\]
A local block script for $v$ is a finite word over the alphabet
\[
\calA_{\mathrm{lo}}
\sqcup
\{\operatorname{call}(w):w\in\operatorname{Child}_M(v)\}.
\]
Let $\mathsf{LCert}_M(v)$ be the set of such scripts that pass the local verifier check for the interface $\lambda(v)$ when each call $\operatorname{call}(w)$ is treated as an available lemma or solved subgoal with interface $\lambda(w)$.
A high-level certificate over $M$ is a tuple
\[
C=(c_v)_{v\in V}
\in
\prod_{v\in V}\mathsf{LCert}_M(v).
\]
The space of high-level certificates over $M$ is denoted $\mathsf{Cert}_{\mathrm{hi}}(M)$, and the total typed domain of unfolding is the disjoint union
\[
\operatorname{Dom}(\operatorname{Unf})
:=
\bigsqcup_{M\in\Latent}
\bigl(\{M\}\times\mathsf{Cert}_{\mathrm{hi}}(M)\bigr).
\]
The unfolding map is
\[
\operatorname{Unf}:
\operatorname{Dom}(\operatorname{Unf})
\to
\calA_{\mathrm{lo}}^\ast.
\]
It is defined recursively along any reverse topological order of $M$:
in $c_v$, each symbol $\operatorname{call}(w)$ is replaced by the already unfolded low-level word for $w$.
The output $\operatorname{Unf}(M,C)$ is the unfolded word for the root block $r$.
The verifier accepts a high-level certificate only if the ordinary low-level rollout
\[
F_{\mathrm{lo}}^\ast(\jmath(\lambda(r)),\operatorname{Unf}(M,C))
\]
is a solved state in $G$.
\end{definition}

\subsection{A cut-elimination example}

The next example instantiates the definitions above in the concrete
cut-elimination family of \cref{sec:example-cut}.
It also records where memoization enters the decision-count comparison used in
\cref{ex:two-m}.

\begin{example}[High-level structure in the cut-elimination example]
\label{ex:app-hilevel-cut}
Consider the sequent-calculus construction in \cref{sec:example-cut}.
The underlying low-level proof system is the ordinary sequent calculus for
$\to$ and $\wedge$.
Thus a low-level proof state is a finite collection of open sequents, and a
low-level action is an inference-rule application such as $\to L$, $\wedge R$,
or $\mathrm{ax}$.

\paragraph{Interfaces.}
In this example an \emph{obligation} is a sequent judgement
$\Gamma\vdash\varphi$.
The three useful block interfaces are the singleton interfaces
\[
I_Y:=\{(\Gamma,B)\},\qquad
I_Z:=\{(\Gamma,B\wedge B)\},\qquad
I_X:=\{(\Gamma,(B\wedge B)\wedge(B\wedge B))\}.
\]
Here $I_Y$ is the interface for the long proof of $B$ from the implication
chain, $I_Z$ is the interface for building $B\wedge B$ from two uses of $B$,
and $I_X$ is the root interface for the final theorem.
The map $\jmath(I_Y)$ is the ordinary low-level initial state with one open
goal $\Gamma\vdash B$, and similarly for $I_Z$ and $I_X$.

\paragraph{Proof DAG.}
A high-level proof DAG for the cut-aware proof has three nodes
\[
V=\{x,z,y\},\qquad r=x,
\]
with labels
\[
\lambda(y)=I_Y,\qquad
\lambda(z)=I_Z,\qquad
\lambda(x)=I_X,
\]
and edges
\[
E=\{(x,z),(z,y)\}.
\]
Diagrammatically,
\[
x:I_X
\xrightarrow[\text{two calls in }c_x]{(x,z)}
z:I_Z
\xrightarrow[\text{two calls in }c_z]{(z,y)}
y:I_Y .
\]
The edge $z\to y$ says that the local proof of the block $z$ may call the
already solved block $y$; the edge $x\to z$ says that the root block may call
the solved block $z$.
The edge set records dependency types, not call multiplicities.
The multiplicities appear in the local certificates.

\paragraph{Local certificates.}
Let $C=(c_x,c_z,c_y)$ be the high-level certificate.
The local certificate $c_y$ is the ordinary low-level chain
\[
\pi_Y:\Gamma\vdash B
\]
obtained by repeated $\to L$ steps.
The certificate $c_z$ consists of two calls to $y$ followed by a $\wedge R$
step, proving $\Gamma\vdash B\wedge B$.
The certificate $c_x$ consists of two calls to $z$ followed by a $\wedge R$
step, proving the root sequent.
Thus $y,z,x$ are the blocks/nodes, the interfaces $I_Y,I_Z,I_X$ are their
input-output specifications, and the scripts $c_y,c_z,c_x$ are their local
certificates.

\paragraph{Cut-aware proof and high-level calls.}
The proof diagram from \cref{sec:example-cut} that this example focuses on is
the cut-aware, memoized one.
Writing
\[
\Phi:=(B\wedge B)\wedge(B\wedge B),
\]
the low-level cut-aware proof diagram is
\[
\infer[\mathrm{cut}]{\Gamma\vdash \Phi}{
  \pi_Y:\Gamma\vdash B
  &
  \infer[\wedge R]{\Gamma,B\vdash \Phi}{
    \infer[\wedge R]{\Gamma,B\vdash B\wedge B}{
      \infer[\mathrm{ax}]{\Gamma,B\vdash B}{}
      &
      \infer[\mathrm{ax}]{\Gamma,B\vdash B}{}
    }
    &
    \infer[\wedge R]{\Gamma,B\vdash B\wedge B}{
      \infer[\mathrm{ax}]{\Gamma,B\vdash B}{}
      &
      \infer[\mathrm{ax}]{\Gamma,B\vdash B}{}
    }
  }
}.
\]
The high-level layer abstracts this proof by replacing the repeated uses of the
assumption $B$ by calls to reusable blocks.
At the high-level call boundary it is represented by the following three local
certificates:
\begin{align*}
c_y
&=
\pi_Y:\Gamma\vdash B,\\
c_z
&=
\infer[\wedge R]{\Gamma\vdash B\wedge B}{
  \operatorname{call}(y):\Gamma\vdash B
  &
  \operatorname{call}(y):\Gamma\vdash B
},\\
c_x
&=
\infer[\wedge R]{\Gamma\vdash (B\wedge B)\wedge(B\wedge B)}{
  \operatorname{call}(z):\Gamma\vdash B\wedge B
  &
  \operatorname{call}(z):\Gamma\vdash B\wedge B
}.
\end{align*}
Here $\pi_Y$ is the long sequent-calculus derivation of
$\Gamma\vdash B$.
The symbols $\operatorname{call}(y)$ and $\operatorname{call}(z)$ are not
low-level inference rules; they are high-level certificate calls that are later
expanded by $\operatorname{Unf}$.

\paragraph{High-level search state.}
One possible high-level action sequence is
\[
\begin{aligned}
(M,\{x,z,y\},\emptyset,\emptyset)
&\xrightarrow{\operatorname{close}(y,c_y)}
(M,\{x,z\},\{y\},\{y\mapsto c_y\})\\
&\xrightarrow{\operatorname{close}(z,c_z)}
(M,\{x\},\{y,z\},\{y\mapsto c_y,z\mapsto c_z\})\\
&\xrightarrow{\operatorname{close}(x,c_x)}
(M,\emptyset,\{x,z,y\},C).
\end{aligned}
\]

\paragraph{Memoization and flattening.}
The proof of $z$ uses $\operatorname{call}(y)$ twice, and the proof of $x$
uses $\operatorname{call}(z)$ twice.
In the hierarchical execution, once $\operatorname{close}(y,c_y)$ has been
performed, the solved set $R_t$ and certificate table $C_t$ store the block
$y$ and its certificate.
Later calls to $y$ consume this stored certificate and do not create new
low-level training occurrences.
Similarly, a graph-construction step $\operatorname{reuse}(v,w)$ discharges a
compatible unresolved block $v$ by an already solved block $w$ rather than
duplicating the proof of $w$.
Flattening removes this memo table:
each call to a block is replaced by a fresh copy of that block's unfolded
low-level proof.
Consequently $\operatorname{Unf}(M,C)$ contains four copies of the long chain
$c_y$.
This is the concrete reason why Setting~3 distinguishes the high-level
decision count $L_{\mathrm{hi}}$, the unique-block low-level count
$L_{\mathrm{lo}}$, and the occurrence-level flat count $L_{\mathrm{flat}}$.
\end{example}

\section{Concrete Example on How to Count High- and Low-Level Decisions}
\label{ex:two-m}
\label{sec:how-to-count-L}

This appendix spells out the decision-count accounting used in
\cref{thm:sample-neurips}.
It uses the same cut-elimination example as \cref{sec:example-cut} and the
same high-level notation as \cref{ex:app-hilevel-cut}, but the relevant proof
diagrams are repeated here so that the counting can be read independently.
Capital letters such as $X,Y,Z$ denote the informal theorem or lemma names in
\cref{sec:example-cut}; lowercase letters $x,y,z$ denote nodes of the
high-level proof DAG.

\paragraph{The base proof.}
Let
\[
\Gamma
=\{A,A\to B_1,B_1\to B_2,\ldots,B_6\to B\}
\]
and let $\pi_Y:\Gamma\vdash B$ be the implication-chain proof of $B$.
The chain has
\[
\ell=7
\]
critical low-level implication steps.
For the depth-two target
\[
\Phi:=(B\wedge B)\wedge(B\wedge B),
\]
the cut-aware proof derives $\pi_Y$ once and then reuses it:
\[
\infer[\mathrm{cut}]{\Gamma\vdash \Phi}{
  \pi_Y:\Gamma\vdash B
  &
  \infer[\wedge R]{\Gamma,B\vdash \Phi}{
    \infer[\wedge R]{\Gamma,B\vdash B\wedge B}{
      \infer[\mathrm{ax}]{\Gamma,B\vdash B}{}
      &
      \infer[\mathrm{ax}]{\Gamma,B\vdash B}{}
    }
    &
    \infer[\wedge R]{\Gamma,B\vdash B\wedge B}{
      \infer[\mathrm{ax}]{\Gamma,B\vdash B}{}
      &
      \infer[\mathrm{ax}]{\Gamma,B\vdash B}{}
    }
  }
}.
\]
After flattening, the same proof of $B$ is copied to all four leaves:
\[
\infer[\wedge R]{\Gamma\vdash \Phi}{
  \infer[\wedge R]{\Gamma\vdash B\wedge B}{
    \pi_Y:\Gamma\vdash B
    &
    \pi_Y:\Gamma\vdash B
  }
  &
  \infer[\wedge R]{\Gamma\vdash B\wedge B}{
    \pi_Y:\Gamma\vdash B
    &
    \pi_Y:\Gamma\vdash B
  }
}.
\]

\paragraph{The high-level DAG.}
The reusable high-level representation has three block interfaces
\[
I_Y:=\{(\Gamma,B)\},\qquad
I_Z:=\{(\Gamma,B\wedge B)\},\qquad
I_X:=\{(\Gamma,\Phi)\}.
\]
The high-level proof DAG is
\[
M=(V,E,r,\lambda),
\qquad
V=\{x,z,y\},\quad r=x,
\]
with
\[
\lambda(y)=I_Y,\qquad \lambda(z)=I_Z,\qquad \lambda(x)=I_X,
\qquad
E=\{(x,z),(z,y)\}.
\]
Diagrammatically,
\[
x:I_X
\xrightarrow[\text{two calls in }c_x]{(x,z)}
z:I_Z
\xrightarrow[\text{two calls in }c_z]{(z,y)}
y:I_Y .
\]
The edge set records dependency types, while the two-use multiplicities are in
the local certificates:
\[
\begin{aligned}
c_y
&=\pi_Y:\Gamma\vdash B,\\
c_z
&=
\infer[\wedge R]{\Gamma\vdash B\wedge B}{
  \operatorname{call}(y):\Gamma\vdash B
  &
  \operatorname{call}(y):\Gamma\vdash B
},\\
c_x
&=
\infer[\wedge R]{\Gamma\vdash \Phi}{
  \operatorname{call}(z):\Gamma\vdash B\wedge B
  &
  \operatorname{call}(z):\Gamma\vdash B\wedge B
}.
\end{aligned}
\]
A compatible high-level search trace is
\[
\begin{aligned}
(M,\{x,z,y\},\emptyset,\emptyset)
&\xrightarrow{\operatorname{close}(y,c_y)}
(M,\{x,z\},\{y\},\{y\mapsto c_y\})\\
&\xrightarrow{\operatorname{close}(z,c_z)}
(M,\{x\},\{y,z\},\{y\mapsto c_y,z\mapsto c_z\})\\
&\xrightarrow{\operatorname{close}(x,c_x)}
(M,\emptyset,\{x,z,y\},C).
\end{aligned}
\]

\paragraph{Counting convention for the displayed numbers.}
The main-text comparison uses decision counts for the chosen protocol.
For this numerical example, we count one high-level decision for each unique
block-closing action above.
Thus
\[
L_{\mathrm{hi}}=3.
\]
The low-level count records the expensive ordinary tactic chain needed to solve
the unique terminal block $y$, namely
\[
L_{\mathrm{lo}}=\ell=7.
\]
The small conjunction-assembly certificates $c_z$ and $c_x$ are treated as
bounded-size structural overhead in the high-level representation.
If one instead counts their two $\wedge R$ applications as low-level local
actions, then $L_{\mathrm{lo}}$ becomes $\ell+2$ and only constants move between
the two hierarchical counters.
This does not affect the separation.

For the flat learner, the unfolded certificate $\operatorname{Unf}(M,C)$ has
four fresh copies of $c_y$ and three conjunction-combination decisions.
Hence
\[
L_{\mathrm{flat}}=4\ell+3=31.
\]
If one counts only the expensive implication-chain decisions and treats all
conjunction assembly as constant overhead, the same example gives
$L_{\mathrm{lo}}=7$ and $L_{\mathrm{flat}}=28$.
In either convention, the mechanism is the same:
hierarchical proving stores the block $y$ once, whereas flattening pays for
four occurrences of the same proof of $B$.

\paragraph{Depth-\texorpdfstring{$d$}{d} family.}
For the recursive family
\[
B^{(0)}:=B,\qquad B^{(j+1)}:=B^{(j)}\wedge B^{(j)},
\]
use nodes $v_0,\ldots,v_d$ with interfaces
$I_j:=\{(\Gamma,B^{(j)})\}$ and edges
\[
v_d\to v_{d-1}\to\cdots\to v_0,
\]
where each certificate $c_{v_j}$ for $j\ge1$ calls $v_{j-1}$ twice and then
applies one conjunction rule.
Under the same convention,
\[
L_{\mathrm{hi}}=d+1,\qquad
L_{\mathrm{lo}}=\ell,\qquad
L_{\mathrm{flat}}=2^d\ell+(2^d-1).
\]
Thus
\[
\frac{L_{\mathrm{flat}}}{L_{\mathrm{hi}}+L_{\mathrm{lo}}}
=
\frac{2^d\ell+(2^d-1)}{\ell+d+1},
\]
which is exponential in $d$ whenever the terminal proof length $\ell$ is fixed
or grows at most polynomially in $d$.
This is the concrete duplication pattern used when
\cref{thm:sample-neurips} assumes
$L_{\mathrm{flat}}/L_{\mathrm{hier}}=\exp(\Omega(D))$.

\end{document}